\documentclass[12pt]{article} 
\usepackage[sectionbib]{natbib}
\usepackage{array,epsfig,fancyheadings,rotating}
\usepackage[]{hyperref}  
\usepackage{sectsty, secdot}
\usepackage{hyperref}
\sectionfont{\fontsize{12}{14pt plus.8pt minus .6pt}\selectfont}
\renewcommand{\theequation}{\thesection\arabic{equation}}
\subsectionfont{\fontsize{12}{14pt plus.8pt minus .6pt}\selectfont}

\usepackage{amsmath}
\usepackage{amssymb}
\usepackage{amsfonts}
\usepackage{multirow}
\usepackage{amsthm}

\newtheorem{theorem}{Theorem}

\newtheorem{corollary}{Corollary}

\theoremstyle{definition}
\newtheorem{definition}{Definition}

\newtheorem{remark}{Remark}
\usepackage{natbib}

\begin{document}


\renewcommand{\baselinestretch}{2}

\markright{ \hbox{\footnotesize\rm 
}\hfill\\[-13pt]
\hbox{\footnotesize\rm
}\hfill }

\markboth{\hfill{\footnotesize\rm FIRSTNAME1 LASTNAME1 AND FIRSTNAME2 LASTNAME2} \hfill}
{\hfill {\footnotesize\rm Analysis of Noise Contrastive Estimation} \hfill}

\renewcommand{\thefootnote}{}
$\ $\par


\fontsize{12}{14pt plus.8pt minus .6pt}\selectfont \vspace{0.8pc}
\centerline{\large\bf Analysis of Noise Contrastive Estimation }
\vspace{2pt} \centerline{\large\bf from the Perspective of Asymptotic Variance}
\vspace{.4cm}
\centerline{Masatoshi Uehara *,  Takeru Matsuda ** and Fumiyasu Komaki **} \vspace{.4cm} \centerline{\it
* Harvard University}
\centerline{\it
** The University of Tokyo} \vspace{.55cm} 
\fontsize{9}{11.5pt plus.8pt minus
.6pt}\selectfont


\begin{quotation}
\noindent {\it Abstract:}
{\bf Contents of the Abstract.}\\
There are many models, often called unnormalized models, whose normalizing constants are not calculated in closed form. Maximum likelihood estimation is not directly applicable to unnormalized models. Score matching, contrastive divergence method, pseudolikelihood, Monte Carlo maximum likelihood, and noise contrastive estimation (NCE) are popular methods for estimating parameters of such models. 
In this paper, we focus on NCE. The estimator derived from NCE is consistent and asymptotically normal because it is an M-estimator. NCE characteristically uses an auxiliary distribution to calculate the normalizing constant in the same spirit of the importance sampling. In addition, there are several candidates as objective functions of NCE. 

We focus on how to reduce asymptotic variance. First, we propose a method for reducing asymptotic variance by estimating the parameters of the auxiliary distribution. Then, we determine the form of the objective functions, where the asymptotic variance takes the smallest values in the original estimator class and the proposed estimator classes. We further analyze the robustness of the estimator.

\vspace{9pt}
\par
\end{quotation}\par

\def\thefigure{\arabic{figure}}
\def\thetable{\arabic{table}}

\renewcommand{\theequation}{\thesection.\arabic{equation}}

\fontsize{12}{14pt plus.8pt minus .6pt}\selectfont

\setcounter{section}{0} 
\setcounter{equation}{0} 

\section{Introduction}

Our objective is to estimate the parameters of unnormalized parametric models. 
Often, the model $\exp(-h(x;\theta))$ is normalized, that is, it satisfies
\begin{align}
\label{eq:normalization}
\int \exp(-h(x;\theta))\mathrm{d}\mu(x)= 1 
\end{align}
where $\theta$ is in a Euclidean parameter space and $\mu$ is a known measure. 
Maximum likelihood estimation (MLE) satisfies consistency and asymptotic efficiency for the estimation of such parametric models when the model includes the true distribution. 

However, often the model is not normalized. In that case, the model can be converted into a normalized model by dividing by the normalizing constant $Z(\theta)$, which is defined as
\begin{align*}
Z(\theta) = \int \exp(-h(x;\theta))\mathrm{d}\mu(x).
\end{align*}
Henceforth, we write the normalized $\exp(-h(x;\theta))$, i.e., $\exp(-h(x;\theta))/Z(\theta)$ as $\overline{\exp(-h(x;\theta))}$.
For MLE to be performed strictly on $\overline{\exp(-h(x;\theta))}$, the normalizing constant must be calculated analytically, in other words, in closed form. 
However, the normalizing constants of complex models, for example, models in independent component analysis \citep{ica}, Markov network \citep{BesagJulian1975SAoN}, Boltzmann machine \citep{HintonGeoffreyE.2002TPoE}, truncated distribution \citep{JohnsonNormanLloyd1970Cud}, and exponential-polynomial distribution \citep{HayakawaJumpei2016Eoed}
are not solved in closed form or difficult to compute. We call such models unnormalized models. 

Several methods have been proposed for the estimation of parameters of unnormalized models. One of the popular approaches is a Monte Carlo maximum likelihood \citep{geyer,GeyerC1994Otco} and noise contrastive estimation (NCE) \citep{noise}. Both of them characteristically use auxiliary distributions in the same spirit of the importance sampling. They were extended to a more general class by \cite{noise2}. To avoid confusion, we call estimators of the former type original NCE and call those of the latter type NCE. NCE does not require Markov Chain Monte Carlo (MCMC) when calculating the gradient, resulting in a short computation time. However, the performance depends highly on the choice of objective function and auxiliary distribution.
When these choices are poor, the variance of estimators becomes too large.

In this study, we consider how to reduce the asymptotic variance in NCE from several points of view. First, we propose a method for reducing asymptotic variances by plugging in the MLE estimator into the parameters of the auxiliary distribution. Second, we analyze the type of objective functions to be used from the perspective of asymptotic variance. We determine the form of the objective function minimizing the asymptotic variance in the class of original estimators and the class of our proposed estimators. 
Finally, we analyze the robustness based on the influence function obtained in the analysis of asymptotic variance.

There are three other popular methods for estimating parameters of unnormalized models, score matching, contrastive divergence method and pseudolikelihood. First, score matching is known to be a proper scoring rule \citep{DawidA.Philip2012PLSR, dawid2}, which does not require integration of the model. Score matching has the advantage of being fast. However, the asymptotic variance is generally large compared with NCE. Second, contrastive divergence method is MLE using Markov chain Monte Carlo (MCMC) when calculating the gradient of the log-likelihood approximately \citep{con, younes}. However, the time required for MCMC is much longer than NCE. Third, pseudolikelihood is a popular method in some models such as Ising model \citep{BesagJulian1975SAoN}. However, it cannot be applied to many unnormalized models directly. 

The remainder of this article is organized as follows.
Section 2 gives a brief review of NCE. 
Section 3 introduces our improved NCE and its asymptotic results. Section 4 discusses which form of objective function leads to the smallest asymptotic variance. Section 5 deals with the robustness. Section 6 contains experimental results. Section 7 includes a briefly summary and future directions. All of proofs are included in an appendix. 

\section{Preliminary: Noise contrastive estimation}

NCE was originally proposed by \cite{noise} as a viewpoint of classification to estimate the parameters of unnormalized models. Here, we summarize NCE from a different perspective, that is, divergence minimization, framework because it provides a more unified view \citep{eguchi,noise2, hirayama}.

As mentioned in Section 1, the objective is to estimate the true $\theta$ in $\exp(-h(x;\theta))$. In NCE, we introduce a one-parameter extended model $p(x;\alpha)$ 
\begin{align}
\label{eq:un}
p(x;\alpha)= \exp \left (c-h(x;\theta)\right ),\, \alpha =\left(c,\theta^{\top}\right)^{\top}.
\end{align}
Both of $x$ and $\theta$ are finite dimensional real vectors, and $c$ is a real positive value. 
We denote the true probability density function as $g(x)$ and the true parameter value as $\alpha^{*}=(c^{*},\theta^{*})$, that is, $g(x)=p(x;\alpha^{*})$. Again, importantly, it is not assumed that the model $p(x;\alpha)$ is normalized.

\subsection{Divergence and cross entropy}

Let $H(s)$, which is called the entropy, be a strictly convex functional, mapping a real-valued function $s$ to a real value. 
A strictly convex of functional is defined as follows.
\begin{definition}
The functional $H(s)$ defined on a convex set is said to be a strictly convex if
\begin{align*}
H\left(\beta s_{1}+(1-\beta)s_{2} \right)< \beta H\left (s_{1}\right)+\left (1-\beta \right)H\left (s_{2}\right )\,(0< \beta < 1).
\end{align*}
\end{definition}

The induced divergence, also called the Bregman divergence, between $g(x)$ and $p(x;\alpha)$ from the entropy $H(s)$ is defined as
\begin{align}
\label{eq:cross}
D_{B}(g,p) = H(g)-H(p)-\int \nabla_{p} H(p)\left (g(x)-p(x)\right )\mathrm{d}\mu(x), 
\end{align}
when $\nabla_{p}$ denotes the differentiation with respect to $p$. The cross entropy between $g$ and $p(x;\alpha)$ is defined as 
\[ d_{B}(g,p) = -H(p)-\int \nabla_{p} H(p;x)\left (g(x)-p(x)\right )\mathrm{d}\mu(x).\]
Our objective is to estimate $\alpha^{*}$. This is the same as the minimization problem of $D_{B}(g,p(x;\alpha))$ with respect to $\alpha$ because $D_{B}(g,p)\geq0$ holds and $D_{B}(g,p)=0$ if and only if $g=p$. This problem is also the same as the minimization problem of cross entropy since we have $d_{B}(g,p)-d_{B}(g,g)=D_{B}(g,p)$ and $d_{B}(g,g)$ is constant.

Empirical minimization of the above cross entropy leads to an estimator of $\alpha$: $d_{B}\left(\hat{g},p(x;\alpha)\right)$, where $\hat{g}$ is an empirical distribution.
When $H(s)$ is $\int s(x)\log s(x)\mathrm{d}\mu(x)$, the induced divergence is called the generalized Kullback–Leibler divergence, as explained in later in connection with \eqref{eq:KL}.
When $H(s) = \frac{1}{\beta (1+\beta)}\int s(x)^{1+\beta} \mathrm{d}\mu(x)$, the induced divergence is called a density-power divergence or beta-divergence \citep{basu,info}. When $H(s) = \int f\left(s(x)\right)\mathrm{d}\mu(x)$ such that $f(x)$ is a strictly convex function \citep{eguchi}, the divergence is called U-divergence. This includes generalized Kullback–Leibler divergence and beta-divergence. In the case of U-divergence, the empirical cross entropy becomes 
\begin{align}
\label{eq:crosscross}
 d_{B}(\hat{g},p) = -\frac{1}{m_{1}}\sum_{i=1}^{m_{1}} f'\left(p(x_{i})\right)+\int \left(f'\left (p(x)\right)p(x)-f\left(p(x)\right)\right) \mathrm{d}\mu(x),
\end{align}
where $(x_{1},x_{2},\cdots,x_{m_{1}})$ is an identically independent distributed (i.i.d) sample from the true distribution $g(x)$. In addition, 
\begin{align*}
D_{B}(g,p)=
\int \mathrm{Br}_{f}\left (g(x)\|p(x;\alpha)\right)\mathrm{d}\mu(x),
\end{align*}
where $\mathrm{Br}_{f}(x\|y)$ is given by $f(x)-f(y)-f'(y)(x-y)$.
We use this property later.
When the divergence is generalized Kullback-Leibler divergence, then \eqref{eq:crosscross} is 
\begin{align}
\label{eq:KL}
 d_{B}(\hat{g},p) = -\frac{1}{n} \sum_{i=1}^{n} \log(p(x_{i}))+\int p(x) \mathrm{d}\mu(x). 
\end{align}
When $p(x;\alpha)$ is normalized, the estimator derived from the above objective function is the same as that of MLE because the term $\int p(x;\alpha) \mathrm{d}\mu(x)$ is constant and can be ignored. However, in our situation, since $p(x;\alpha)$ is not normalized, the induced divergence represented by \eqref{eq:KL} has an extra term $\int p(x;\alpha) \mathrm{d}\mu(x)$.
When the divergence is a density-power or beta-divergence, then \eqref{eq:crosscross} is
\begin{align}
\label{eq:beta}
d_{B}(\hat{g},p) = -\frac{1}{n} \sum_{i=1}^{n} \frac{1}{\beta}p(x_{i}) ^{\beta}+\int \frac{1}{\beta+1}p(x)^{\beta+1} \mathrm{d}\mu(x).
\end{align}
The estimator from beta-divergence is known to be robust from the viewpoint of influence functions \citep{basu}. 

It appears that $\alpha$ can be estimated easily using the \eqref{eq:crosscross}.
However, the difficulty arises because the integral term, $\int \left(f'\left (p(x)\right)p(x)-f\left(p(x)\right)\right) \mathrm{d}\mu(x)$ in \eqref{eq:crosscross} is often not solved analytically in closed form. NCE solves this problem, as explained in the next section. 

\subsection{Noise contrastive estimation}
\label{noisecontrastive}

\cite{noise} proposed the method of NCE. Later, its extension was proposed by \cite{noise2}. 

We define the noise contrastive divergence as the induced divergence when the entropy $H(s)$ is defined as 
\begin{align*}
H(s) = \int f\left (\frac{s(x)}{n(x)}\right )n(x) \mathrm{d}\mu(x),
\end{align*}
where $f(x)$ is a strictly convex function and $n(x)$ is a probability density function. We called $n(x)$ an auxiliary distribution. The noise contrastive divergence from the above entropy is written as 
\begin{align*}
\int \mathrm{Br}_{f}\left (\frac{g(x)}{n(x)}\|\frac{p(x;\alpha)}{n(x)}\right)n(x)\mathrm{d}\mu(x).
\end{align*}
The cross entropy between $g(x)$ and $p(x;\alpha)$ is given by 
\begin{align*}
 d_{B}(g,p) =-\int f'\left (\frac{p}{n}\right )g(x)\mathrm{d}\mu(x)+\int \left (f'\left (\frac{p}{n}\right )\frac{p}{n}-f\left (\frac{p}{n}\right )\right)n(x)\mathrm{d}\mu(x),
\end{align*}
and the empirical cross entropy is written as 
\begin{align*}
 d_{B}(\hat{g},p_{\alpha}) =-\sum _{i=1}^{m_{1}}f'\left(\frac{p(x_{i})}{n(x_{i})}\right)+\int \left(f'\left (\frac{p}{n}\right )\frac{p}{n}-f\left(\frac{p}{n}\right)\right)n(x)\mathrm{d}\mu(x).
\end{align*}
This becomes an objective function for estimating $\alpha$ in $p(x;\alpha)$.
Even if calculating the integral term analytically is difficult, it is easy to calculate the integral term approximately if $n(x)$ is a distribution that is easy to sample. 

In this case, the objective function becomes 
\begin{align}
\label{eq:noi}
 -\frac{1}{m_{1}}\sum_{i=1}^{m_{1}} f'\left (\frac{p(x_{i};\alpha)}{n(x_{i})}\right )+ \frac{1}{m_{2}}\sum _{i=1}^{m_{2}} \left (\frac{p(y_{i};\alpha)}{n(y_{i})}f'\left(\frac{p(y_{i};\alpha)}{n(y_{i})}\right)-f\left (\frac{p(y_{i};\alpha)}{n(y_{i})}\right )\right),
\end{align}
where $(x_{1},\cdots,x_{m_{1}})$ is an i.i.d sample from the true distribution and $(y_{1},\cdots, y_{m_{2}})$ is an i.i.d sample from the probability density function $n(y)$. Henceforth, we call the estimation method represented as the minimization of \eqref{eq:noi} NCE. 
The estimator minimizing \eqref{eq:noi} is consistent under some proper conditions because it fits into the form of M-estimators.
When $f(x)=x\log x$ corresponding Kullback-Lebiler divergence, \eqref{eq:noi} is written as
\begin{align}
\label{eq:ex1}
 -\frac{1}{m_{1}}\sum_{i=1}^{m_{1}} \log\left(\frac{p(x_{i};\alpha)}{n(x_{i})}\right)+ \frac{1}{m_{2}}\sum _{i=1}^{m_{2}}\frac{p(y_{i};\alpha)}{n(y_{i})}.
\end{align}
When $f(x) = 0.5x^{2}$ corresponding to chi-square divergence, \eqref{eq:noi} is written as 
\begin{align}
\label{eq:ex2}
 -\frac{1}{m_{1}}\sum_{i=1}^{m_{1}} \frac{p(x_{i};\alpha)}{n(x_{i})}+ \frac{1}{2m_{2}}\sum _{i=1}^{m_{2}} \left(\frac{p(y_{i};\alpha)}{n(y_{i})}\right )^{2}.
\end{align}
When $f(x)=x\log x-(1+x)\log(1+x)$ corresponding to Jensen-Shannon divergence, \eqref{eq:noi} is written as 
\begin{align}
\label{eq:original}
 -\frac{1}{m_{1}}\sum_{i=1}^{m_{1}} \log \frac{\frac{p(x_{i};\alpha)}{n(x_{i})}}{1+\frac{p(x_{i};\alpha)}{n(x_{i})}} - \frac{1}{m_{2}}\sum _{i=1}^{m_{2}} \log \frac{1}{1+\frac{p(y_{i};\alpha)}{n(y_{i})}}.
\end{align}
This case corresponds to the original NCE first proposed \cite{noise}. 

\cite{noise2} analyzed the general NCE using the \eqref{eq:noi}. They calculated the mean square error and showed many simulation results pointing out that, experimentally, the original noise contrastive estimation is better than other forms of NCE, such as the estimations from \eqref{eq:ex1} and \eqref{eq:ex2}. 
However, they did not show which $f$ in the \eqref{eq:noi} is theoretically the best in the sense of asymptotic variance. We discuss this point later in Section \ref{important}. Before dealing with this, we introduce our improved NCE in the next section. 

\section{Improved noise contrastive estimation}

We introduce improved NCE and show the asymptotic results of the proposed method. The asymptotic variance of the estimator derived from the method is significantly less than of the original estimator.
 
\subsection{Setting}

We explain NCE again in a more formal manner and proposed a new estimator. Assume that the true distribution's density is given by
\begin{align*}
\overline{\exp \left (-h(x;\theta^{*})\right )}=\frac{\exp \left (-h(x;\theta^{*})\right ) }{Z(\theta^{*})},
\end{align*}
where $Z(\theta) = \int \exp \left (-h(x;\theta)\right )\mathrm{d}\mu(x)$. In this case, the model includes a true distribution. We do not consider a model misspecification case. The objective is to estimate the true parameter $\theta^{*}$ from a random sample $(x_{1},\cdots,x_{m_{1}})$ of size $m_{1}$. The problem is that calculating in closed form $Z(\theta)$ is intractable. Consider a one-parameter extended model 
\begin{align*}
p(x;\alpha)=\exp \left(c-h(x;\theta)\right),
\end{align*}
where $\alpha = (c,\theta^{\top})^{\top}$.
When $\alpha$ is equal to $\alpha^{*}=(c^{*},\alpha^{*})$, the model $p(x;\alpha)$ is equivalent to a true density $\overline{\exp \left (-h(x;\theta^{*})\right )}$.
In addition to the samples from the true distribution, suppose that we can use samples from an auxiliary distribution with density $n(x)$. 
We set the number of samples from the auxiliary distribution $m_{2}$ and denote the sample as $(y_{1},\cdots,y_{m_{2}})$.
The natural choice of $n(x)$ is obtained by first considering a model $n(x;\beta)$ parameterized by $\beta$ and choosing $\beta^{*}$ by some methods like moment matching or MLE. For example, we can use the normal distribution family with a parameter and variance. In a real situation, the value $\beta^{*}$ might depends on the data $\{x_{i}\}_{i=1}^{m_{1}}$. However, we assume that the value $\beta^{*}$ does not depends on data for ease of further analysis. This assumption means that the auxiliary density $n(x;\beta^{*})$ is set without looking the data. Other studies also assume this situation \citep{noise,noise2}. Note that we must consider a model such that the support of $n(x;\beta)$ includes the support of $p(x;\alpha)$ for the validity of the estimator. 

As in \eqref{eq:noi}, the estimator for $\alpha$ in the NCE, $\hat{\alpha}_{NC}$, is defined as the minimzer of the following function with respect to $\alpha$:
\begin{align}
\label{eq:loss}
-\frac{m_{2}}{m}\sum_{i=1}^{m_{1}} f'\left(\frac{p(x_{i};\alpha)}{n(x_{i};\beta^{*})}\right)+ \frac{m_{1}}{m} \sum _{i=1}^{m_{2}} \left(\frac{p(y_{i};\alpha)}{n(y_{i};\beta^{*})}f'\left(\frac{p(y_{i};\alpha)}{n(y_{i};\beta^{*})}\right)-f\left(\frac{p(y_{i};\alpha)}{n(y_{i};\beta^{*})}\right)\right). 
\end{align}

We consider another method, estimating $\beta$ again using MLE based on the sample  $\{y_{i}\}_{i=1}^{m_{2}}$ and then plugging in the estimated value into the \eqref{eq:loss}. The estimator $\hat{\alpha}_{PL}$ is written as the minimizer of the following function with respect to $\alpha$:
\begin{align}
\label{eq:est2} 
   -\frac{m_{2}}{m}\sum_{i=1}^{m_{1}} f'\left(\frac{p(x_{i};\alpha)}{n(x_{i};\hat{\beta})}\right)+ \frac{m_{1}}{m}\sum _{i=1}^{m_{2}} \left(\frac{p(y_{i};\alpha)}{n(y_{i};\hat{\beta})}f'(\frac{p(y_{i};\alpha)}{n(y_{i};\hat{\beta})})-f\left(\frac{p(y_{i};\alpha)}{n(y_{i};\hat{\beta})}\right)\right), 
\end{align}
where $\hat{\beta}$ is an MLE estimate based on sample $\{y_{i}\}_{i=1}^{m_{2}}$. Consequently, the true parameter $\alpha^{*}$ is estimated in two manners.

The first estimator was proposed by \cite{noise2}. The second estimator is our proposed estimator. In fact, the asymptotic variance of the second type of estimator is less than that of the first type, as explained in the next section. 
This plug-in method contributes to a reduction in the asymptotic variance. Note that a similar phenomenon appears in semiparametric models and importance sampling \citep{henmi1, henmi2}. 

Beginning in the next section, we discuss the asymptotic behavior of the two estimators. To that end, let us clarify the sampling mechanism we will consider throughout the remainder of this paper. 
We assume standard stratified sampling mechanism \citep{wood}, that is, we draw a random samples of size $m_{1}$ from the true distribution with density $p(x;\alpha^{*})$ and of size $m_{2}$ from the auxiliary distribution with density $n(x;\beta^{*})$. This is different from the sampling mechanism of drawing from a  mixture of distribution $m_{1}/m\times g^{*}(x)+m_{2}/m\times n(x;\beta^{*})$ independently even if the resulting likelihood is the same because the stratified sampling is not i.i.d sampling. This distinction is important especially when considering the asymptotic variance because the results will differ depending on sampling mechanism assumptions. 

We summarize the notations frequently used here. Hereafter, let $\mathrm{E}_{*}[\cdot]$ be an expectation with respect to the stratified sampling mechanism. 
We denote $p(x;\alpha^{*})$ as $p^{*}$, $n(x;\beta^{*})$ and $n^{*}$.
Notations $\mathrm{E}_{p^{*}}[\cdot]$ and $\mathrm{E}_{n^{*}}[\cdot]$ denote expectations with respect to distributions with density $p^{*}(x)$ and $n^{*}(x)$. 
Notations $\mathrm{var}_{p}[\cdot]$, $\mathrm{cov}_{p}[\cdot]$ denote the variance and covariance respectively, when the underlying distribution's density is $p(x)$. The notation $\mathrm{Asvar}[\cdot]$ denotes the asymptotic variance of estimators scaled by sample size $m$.
The notation $d_{\alpha}$ denotes the dimension of $\alpha $, and $\Theta_{\alpha}$ denotes the parameter space of $\alpha$. The notation $\nabla_{\alpha}H(x;\alpha)$ denotes differentiation with respect to $\alpha$, that is, $\left(\frac{\partial H(x;\alpha)}{\partial \alpha_{1}},\frac{\partial H(x;\alpha)}{\partial \alpha_{2}},\cdots,\frac{\partial H(x;\alpha)}{\partial \alpha_{d_{\alpha}}}\right)^{\top}$. Similarly, the notation $\nabla_{\alpha^{\top}}H(x;\alpha)$ is $\left(\frac{\partial H(x;\alpha)}{\partial \alpha_{1}},\frac{\partial H(x;\alpha)}{\partial \alpha_{2}},\cdots,\frac{\partial H(x;\alpha)}{\partial \alpha_{d_{\alpha}}}\right)$. 
The notation $\cdot|_{\alpha=\alpha^{*}}$ denotes evaluation of a function of $\alpha$ at $\alpha^{*}$. For simplicity, we often denote this as $\cdot|_{\alpha^{*}}$.

\subsection{Consistency and asymptotic normality}

Let us calculate the asymptotic variances of $\hat
{\alpha}_{\mathrm{NC}}$ and $\hat{\alpha}_{\mathrm{PL}}$ based on the theory of Z-estimator \citep{van}. To obtain the form of a Z-estimator, assuming the differentiability of $p(x;\alpha)$ and $n(x;\beta)$ with respect to $\alpha$ and $\beta$ respectively, we differentiate the minimized term in \eqref{eq:est2} with respect to $\gamma$, which is defined as $\gamma\equiv (\alpha^{\top},\beta^{\top})^{\top}$. The following estimating equation is obtained: 
\begin{align}
\label{eq:vvv2}
V_{m}(\mathbf{x},\mathbf{y};\gamma)= \begin{pmatrix} V_{1m}(\mathbf{x},\mathbf{y};\gamma) \\ V_{2m}(\mathbf{y};\beta) 
\end{pmatrix}=0,
\end{align}
where $\mathbf{x}$ and $\mathbf{y}$ denote $\{x_{i}\}_{i=1}^{m_{1}}$ and $\{y_{i}\}_{i=1}^{m_{2}}$ and the respective components are 
\begin{align}
\label{eq:first}
V_{1m}(\mathbf{x},\mathbf{y};\gamma)  &=\frac{1}{m}\left(\frac{m_{2}}{m}\sum_{i=1}^{m_{1}} \phi\left(x_{i};\alpha,\beta \right)-\frac{m_{1}}{m} \sum_{i=1}^{m_{2}} \phi\left(y_{i};\alpha,\beta \right)\frac{p(y_{i};\alpha)}{n(y_{i};\beta)}\right), \\
V_{2m}(\mathbf{y};\beta) &= \frac{1}{m}\times \frac{m_{1}}{m}\left(\sum_{i=1}^{m_{2}} -\nabla_{\beta}\log n(y_{i};\beta)\right), \\ 
\phi\left(x;\gamma \right) &= f''\left(\frac{p(x;\alpha)}{n(x;\beta)}\right)\frac{\nabla_{\alpha} p(x;\alpha)}{n(x;\beta)} \label{eq:fff}.
\end{align} 
The estimator $\hat{\alpha}_{PL}$ is defined as the value satisfying the \eqref{eq:vvv2}. The estimator $\hat{\alpha}_{NC}$ is defined as the value satisfying the equation $\mathrm{E}_{*}[V_{1m}(\mathbf{x},\mathbf{y};\alpha,\beta^{*})]=0$.

The equation \eqref{eq:vvv2} takes the form of a Z-estimator, since we have $\mathrm{E}_{*}[V_{m}(\gamma)|_{\gamma^{*}}]=0$ from
\begin{align*}
    \mathrm{E}_{*}[V_{1m}(\mathbf{x},\mathbf{y};\gamma)|_{\gamma^{*}}] &= \frac{m_{2}m_{1}}{m}\mathrm{E}_{p^{*}}[\phi(x;\alpha,\beta)|_{\gamma^{*}}]-\frac{m_{2}m_{1}}{m}\mathrm{E}_{n^{*}}\left[\frac{p(y;\alpha)}{n(y;\beta)}\phi(y;\alpha,\beta)|_{\gamma^{*}}\right] \\
    &=\frac{m_{2}m_{1}}{m}\mathrm{E}_{p^{*}}[\phi(x;\alpha,\beta)|_{\gamma^{*}}]-\frac{m_{2}m_{1}}{m}\mathrm{E}_{p^{*}}[\phi(x;\alpha,\beta)|_{\gamma^{*}}]= 0,
\end{align*}
and $\mathrm{E}_{*}[V_{2m}(\mathbf{x};\beta)|_{\beta^{*}}]=0$ from standard MLE theory. Owing to the theory of Z-estimators, the consistency of $\hat{\alpha}_{PL}$ holds under suitable conditions. Note that it is proved similarly proved that the estimator $\hat{\alpha}_{NC}$ converges in probability to $\alpha^{*}$.

\begin{theorem}[Consistency]
\label{thm3}
Assume that the following conditions hold: (1) $p(x;\alpha)=p(x;\alpha^{*})\iff\alpha=\alpha^{*}$,\,$ n(x;\beta)=n(x;\beta^{*}) \iff \beta=\beta^{*}$; (2) $\Theta_{\gamma}$ is compact;
(3) $\alpha \to p(x;\alpha)$ and $\beta \to n(x;\beta)$ are $\mathrm{C}^{1}$ functions; (4) $V_{m}(\mathbf{x},\mathbf{y};\gamma)$ is uniformly bounded by an integrable function; (5) $d_{B}(p(x;\alpha^{*}),p(x;\alpha))$ is convex in $\alpha$, and $\mathrm{E}_{n^{*}}[\log n(x;\beta)]$ is convex in $\beta$, where $\|\cdot\|$ is an Euclidian norm. 
Then, the estimator $\hat{\alpha}_{PL}$ converges in probability to $\alpha^{*}$.   
\end{theorem}

All of these conditions are typically required to prove consistency of Z-estimator, such as those used in \cite{wood}. The key point of proof here is the proving of an identifiability condition $\mathrm{E}_{*}[V_{m}(\mathbf{x},\mathbf{y};\gamma)]=0\iff \gamma=\gamma^{*}$.

There are two issues associated with Theorem \ref{thm1}. First, we have assumed that $p(x;\alpha)$ and $n(x;\beta)$ belongs to $C_{1}$. However, this assumption is a little strong, because some models might not be differentiable with respect to parameters.
This case includes an important Laplacian-based model in independent component analysis \citep{noise}.
In this case, by assuming a differentiability of $\alpha \to \sqrt{p(x;\alpha)}$ in quadratic mean and using the score function derived by differentiability in quadratic mean \citep{van}, we can construct an estimator. The estimator in \eqref{eq:vvv2}
is defined as replacing $\nabla_{\alpha}p(x;\alpha)$ with the redefined score function. Conditions required for the consistency of (3) and (4) in Theorem \ref{thm3} are replaced with (3) $\mathrm{E}_{*}[V_{m}(\mathbf{x},\mathbf{y};\gamma)]$ is continuous in $\gamma$ and (4) $\sup_{\gamma\in \Theta_{\gamma}}\|V_{m}(\mathbf{x},\mathbf{y};\gamma)-\mathrm{E}_{*}[V_{m}(\mathbf{x},\mathbf{y};\gamma)]\|\to 0$.

The second issue is a model misspecification. We have assumed that the posited model includes a true distribution. In other words, we know that $p(x;\alpha)$ evaluated at $\alpha=\alpha^{*}$ is the true distribution, where $\alpha^{*}$ is the part of $\gamma^{*}$ satisfying the equation $\mathrm{E}_{*}[V_{m}(x;\gamma)]=0$. When the posited model does not include the true distribution, the value $\gamma^{*}$ is also defined as the solution to the equation $\mathrm{E}_{*}[V_{m}(\gamma)]=0$ in  $\gamma$. The interpretation of $\alpha^{*}$ becomes a value minimizing a cross entropy $d_{B}(g(x),p(x;\alpha))$ between a true density $g(x)$ and a model $p(x;\alpha)$. The density $p(x;\alpha^{*})$ is no longer equal to $g(x)$.

Next, let us consider the asymptotic variances. As we mentioned, the sampling mechanism is a stratified sampling; hence, the estimator becomes a stratified Z-estimator \citep{wood}.
The asymptotic variance of $\hat{\gamma}_{PL}\equiv (\hat{\alpha}_{PL},\hat{\beta}_{PL})$ becomes $\Omega_{1}^{-1}\Omega_{2}{\Omega_{1}^{\top}}^{-1}$, where 
\begin{align}
\label{eq:sand2}
\Omega_{1} &= \mathrm{E}_{*}\left[\nabla_{\gamma^{\top}}V_{m}(\mathbf{x},\mathbf{y};\gamma)|_{\gamma^{*}}\right] \\
\Omega_{2} &= \frac{m_{1}}{m}\mathrm{var}_{p^{*}}\left[\left(\frac{m_{2}}{m}\phi(x)^{\top},0_{d_{\beta}\times 1}\right)^{\top}|_{\gamma^{*}}\right]+ \\&\frac{m_{2}}{m}\mathrm{var}_{n^{*}}\left[\frac{m_{1}}{m}\left(-\frac{p(y)}{n(y)}\phi(y)^{\top},-\nabla_{\beta}\log n(y;\beta)^{\top}\right)^{\top}|_{\gamma^{*}}\right] \nonumber,
\end{align}
and where $0_{d_{\beta}\times 1}$ is a $d_{\beta}\times 1$ zero matrix.
The resulting asymptotic variance is different from that of an M-estimator with i.i.d observations; however, the conditions require to ensure asymptotic normality are essentially the same. Compared to a case where a set of samples is independently obtained from a mixture distribution, the asymptotic variance becomes smaller owing to the stratification because the sampling mechanism is more conditioned. The asymptotic variance is obtained as follows. 

\begin{theorem}[Asymptotic normality]
\label{thm1}
Assume that (1) $\alpha\to  p(x;\alpha)$ and $\beta \to n(x;\beta)$ are twice continuously differentiable in a neighborhood of $\alpha^{*}$; (2) the Central Limit Theorem holds for $\phi(x_{i})|_{\gamma^{*}}$ when $\{x_{i}\}$ is drawn from a distribution  with density
$p(x;\alpha^{*})$ and for $\phi(y_{i})\frac{p(y_{i})}{n(y_{i})}|_{\gamma^{*}}$ and $\log n(y_{i};\beta)$
when $\{y_{i}\}$ is drawn from a distribution with density $n(y;\beta^{*})$;
(3) $\gamma \to \mathrm{E_{*}}[\nabla_{\gamma}V_{m}(\mathbf{x},\mathbf{y};\gamma)]$ is continuous at $\gamma^{*}$, (4)  $\nabla_{\gamma}V_{m}(\mathbf{x},\mathbf{y};\gamma)$ uniformly converges to  $\mathrm{E}_{*}[\nabla_{\gamma}V_{m}(\mathbf{x},\mathbf{y};\gamma)]$ around a neighrborhood of $\gamma^{*}$; and (5) $\mathrm{E}_{*}[\nabla_{\gamma}V_{m}(\mathbf{x},\mathbf{y};\gamma)]$ is non-singular, In addition, we also assume the assumptions in Theorem \ref{thm3} to ensure consistency.

The estimator $\hat{\alpha}_{PL}$ converges in law to a normal distribution:
\begin{align}
\label{eq:asym}
\sqrt{m}(\hat{\alpha}_{PL}- \alpha^{*}) & \stackrel{d}{\longrightarrow }\mathrm{N}\left(0, A^{-1}\left[G-\frac{m_{1}}{m}BC^{-1}B^{\top}\right]A^{-1}\right), 
\end{align}
where
\begin{align}
\label{eq:comp}
A &= \mathrm{E}_{*}\left[\nabla_{\alpha^{\top}}V_{1m}|_{\gamma = \gamma^{*}}\right]=- \frac{m_{1}m_{2}}{m^{2}}\mathrm{E}_{p^{*}}\left[\phi(x;\gamma)\left(\nabla_{\alpha^{\top}} \log p(x;\alpha)\right)|_{\gamma^{*}}\right], \\ 
B &= \mathrm{E}_{*}\left[\nabla_{\beta^{\top}}V_{1m}|_{\gamma = \gamma^{*}}\right]= 
\frac{m_{1}m_{2}}{m^{2}}\mathrm{E}_{p^{*}}\left[\phi(x;\gamma)\nabla_{\beta^{\top}}\log n(x;\beta)|_{\gamma^{*}}\right],  \\
C &= \mathrm{E}_{*}\left[\nabla_{\beta^{\top}} V_{2m}|_{\gamma = \gamma^{*}}\right]=- \frac{m_{1}m_{2}}{m^{2}}\mathrm{E}_{n^{*}}\left[\nabla_{\beta^{\top}} \nabla_{\beta}\log n(x;\beta)|_{\gamma}\right] =\mathrm{var}_{n^{*}}[n(x;\beta)],
\\
G &= \frac{m_{1}m_{2}}{m^{2}}\left(\mathrm{E}_{ p^{*}}\left[\phi(x;\gamma)\phi(x;\gamma)^{\top}\left(\frac{m_{2}}{m}+\frac{m_{1}}{m}\frac{p}{n}\right)\right]-\mathrm{E}_{p^{*}}[\phi(x;\gamma)]\mathrm{E}_{p^{*}}\left[\phi(x;\gamma)^{\top}\right]\right)|_{ \gamma^{*}}.
\end{align}
\end{theorem}
Associated with Theorem 2, in an appendix, we explain the meaning of conditions and the view from influence functions. 

Theorem \ref{thm1} states that the asymptotic variance of $\hat{\alpha}_{PL}$ is significantly reduced compared with the estimator $\hat{\alpha}_{NC}$. As in Theorem \ref{thm1}, we have 
\begin{align*}
\sqrt{m}(\hat{\alpha}_{NC}- \alpha^{*}) &\stackrel{d}{\longrightarrow } \mathrm{N}(0,A^{-1}GA^{-1}). 
\end{align*}
We can know that the asymptotic variance of $\hat{\alpha}_{PL}$ is less than that of $\hat{\alpha}_{NC}$ because we know that the difference $A^{-1}BC^{-1}B^{\top}A^{-1}$ is a positive-definite matrix. 
\label{reduce}

\subsection{Extension of the noise contrastive estimation}

In Section \ref{reduce}, the function $\phi(x;\gamma)$ indexed by $\gamma$ in \eqref{eq:first} is restricted to the special form as in \eqref{eq:fff}. Without this restriction, for any $\phi(x;\gamma)$, we have $\mathrm{E}_{*}[V_{m}(\mathbf{x},\mathbf{y};\gamma)]=0$ at $\gamma=\gamma^{*}$ when $V_{m}(\mathbf{x},\mathbf{y};\gamma)$ is defined as in \eqref{eq:vvv2}. Asymptotic normality holds and the asymptotic variance is written in the same manner as in Theorem \ref{thm1}.

Based on the above discussion, we consider broadening the class of $\phi(x;\gamma)$ broader. Here, we define a broader class for $\phi(x;\gamma)$, and discuss the asymptotic efficient estimator in this class in the next section. However, we first review the original class we discussed in the previous section. 

\begin{definition}[Naive class of $\phi(x;\gamma)$]
We define the class of $\phi(x;\gamma)$ represented by
\begin{align*}
f''\left(\frac{p(x;\alpha)}{n(x;\beta)}\right) \frac{\nabla_{\alpha}p(x;\alpha)}{n(x;\beta)}
\end{align*}
and denote it by $\mathcal{O}$.
\end{definition}
As mentioned, the objective function can be extended to a broader class:
\begin{definition}[Broader class of $\phi(x;\gamma)$]
\label{broad}
Define the class of $\phi(x;\gamma)$ represented by $\psi(x) \nabla_{\alpha}\log p(x;\alpha)$ as $\mathcal{L}$ where $\psi(x;\gamma)$ is a one-dimensional function of $x$ indexed by $\gamma$ taking positive values.
\end{definition}

There are a few things to note regarding the relationship between $\mathcal{L}$ and $\mathcal{O}$.
First, the Z-estimators shown in \eqref{eq:vvv2} when $\phi(x)\in \mathcal{L}$ are not necessarily represented in the form of M-estimators. On the other hand, when $\phi(x)\in \mathcal{O}$, estimators can be represented in the form of M-estimators by recalling the original derivation of the estimation. 
Second, the class $\mathcal{O}$ is included in class $\mathcal{L}$ because elements in $\mathcal{O}$ can be represented as $\phi(x;\gamma)=\psi(x;\gamma) \nabla_{\alpha} \log p(x;\alpha)$, where 
\begin{align}
\label{eq:phi}
    \psi(x;\gamma) = f''\left(\frac{p(x;\alpha)}{n(x;\beta)}\right)\frac{p(x;\alpha)}{n(x;\beta)}.
\end{align} 
In the next section, we will find the optimal $\phi(x)$ in this broader class $\mathcal{L}$ from the perspective of asymptotic variance. It is shown in later that a function $\phi(x)$ lying in $\mathcal{O}$ reaches the minimum.

\section{Objective functions minimizing asymptotic variances} 

When the parametric model is normalized, the optimal objective function minimizing the asymptotic variance of the estimator is that derived from Kullback-Leibler divergence, that is, MLE. On the other hand, in fact, this does not holds in the case of noise contrastive estimation. Hereafter, we consider objective functions minimizing asymptotic variance. We will see that the estimator from \eqref{eq:original} is best for $\hat{\alpha}_{NC}$ and that the estimator from \eqref{eq:ex1} is the best estimator for $\hat{\alpha}_{PL}$. 

The asymptotic variances $\hat{\alpha}_{PL}$ and $\hat{\alpha}_{NC}$ when $\phi(x)\in\mathcal{L}$ are  
\begin{align*}
\mathrm{Asvar}(\hat{\alpha}_{PL}) = V_{1}-V_{2},
\, \mathrm{Asvar}(\hat{\alpha}_{NC}) = V_{1},
\\V_{1} = A^{-1}GA^{-1},\, V_{2}=\frac{m_{1}}{m}A^{-1}BC^{-1}B^{\top}A^{-1}, 
\end{align*}
respectively, using the notations in Theorem \ref{thm1}.
Here, we first consider only $V_{1}$ only and find the minimum the asymptotic variance of $\hat{\alpha}_{NC}$ and the objective function reaching that minimum. The term $V_{2}$ will be considered later. We also mention on the hypothesis testing and the confidence intervals for $\alpha$. 

We summarize new notations we will use frequently. Let the ratio be $r(x;\gamma)=p(x;\alpha)/n(x;\beta)$. Suppose that $\phi$ belongs to the broader class $\mathcal{L}$ represented as $\psi(x)(1, -\frac{\partial h}{\partial \theta}^{\top})^{\top}$, as introduced in the Section 3.3, noting that $p(x;\alpha)=\exp(c-h(x;\theta))$
We write $N\preceq M$ when $M-N$ is positive semi-definite. 

\subsection{Optimal \texorpdfstring{$\phi(x;\alpha)$}{phi(x;alpha)} minimizing the asymptotic variance of \texorpdfstring{$\hat{\alpha}_{NC}$}{n(x;beta)}}

We consider the minimum asymptotic variance and the form of estimators $\hat{\alpha}_{NC}$ in the class of $\mathcal{L}$. The parameter $\beta$ is fixed at $\beta^{*}$. The optimal objective function minimizing the asymptotic variance of $\hat{\alpha}_{NC}$ is obtained regardless of $n(x;\beta^{*})$. It is shown that when $m_{1}/m$ is equal to 0.5, the optimal objective function is that of original NCE \eqref{eq:original}. In the general case, the optimal objective function is that from $f(x)=x\log x-\left(\frac{m_{2}}{m_{1}}+x\right)\log(1+\frac{m_{1}}{m_{2}}x)$. Note that when $m_{2}/m_{1}\to\infty$, it goes to $x\log x$. We call the $f$-divergence corresponding to this $f(x)$ as the optimal Jensen-Shannon divergence. In this case, the objective function becomes 
\begin{align}
\label{eq:original2}
 -\frac{1}{m_{1}}\sum_{i=1}^{m_{1}} \log \frac{\frac{p(x_{i};\alpha)}{n(x_{i})}}{1+\frac{m_{1}p(x_{i};\alpha)}{m_{2}n(x_{i})}} - \frac{1}{m_{2}}\sum _{i=1}^{m_{2}} \frac{m_{2}}{m_{1}}\log \frac{1}{1+\frac{m_{1}p(y_{i};\alpha)}{m_{2}n(y_{i})}}.
\end{align}
This fact is based on the following theorem. 
\label{important}

\begin{theorem}[Optimal \texorpdfstring{$\phi(x;\alpha)$}{phi(x;alpha)} minimizing the asymptotic variance of \texorpdfstring{$\hat{\alpha}_{NC}$}{hat nc}]
\label{thm4}
The optimal $\phi(x)$ minimizing the asymptotic variance $V_{1}$ over the class $\mathcal{L}$ satisfies $\frac{m}{m_{2}+m_{1}r(x;\gamma)}\nabla_{\alpha} \log p(x;\alpha)|_{\gamma^{*}}=\phi(x;\gamma)|_{\gamma^{*}}$ up to scale. 
The minimum $V_{1}$ is 
\begin{align}
\label{eq:optimal}
\frac{m^{2}}{m_{1}m_{2}}\left(H^{-1}-\begin{pmatrix} 1 & 0_{1\times d_{\theta}}\\ 0_{d_{\theta}\times 1} & 0_{d_{\theta}\times d_{\theta}}\end{pmatrix}\right),\\ H = \mathrm{E}_{p^{*}}\left[\Omega(x) \frac{1}{\frac{m_{2}}{m}+\frac{m_{1}}{m}r^{*}(x)}\right],\\
\Omega(x) = \begin{pmatrix} 1 & -\nabla_{\theta^{\top}} h^{*}(x)\\  -\nabla_{\theta} h^{*}(x) & \nabla_{\theta} h^{*}(x)\nabla_{\theta^{\top}} h^{*}(x)
\end{pmatrix},
\end{align}
where $\frac{p(x;\alpha)}{n(x;\beta)}|_{\gamma^{*}}=r^{*}(x)$ and $\nabla_{\theta}h(x;\theta)|_{\theta=\theta^{*}}=\nabla_{\theta}h^{*}(x)$. 
The function $\phi(x)$, which belongs to the class $\mathcal{O}$ with $f(x)=x\log x-\left(\frac{m_{2}}{m_{1}}+x\right) \log(1+\frac{m_{1}}{m_{2}}x)$, achieves the lower bound.
\end{theorem}

Next, consider the hypothesis testing for the parameters of interests from Theorem \ref{important}. Hypothesis testing can be conducted and confidence regions can be constructed by using a consistent estimate of asymptotic variance.
Inference for $\alpha$ can be conducted using Wald-type statistics 
\begin{align}
\label{eq:wald}
(\hat{\alpha}_{NC}-{\alpha}^{*})\mathrm{Asvar}\left[\hat{\alpha}_{NC}\right](\hat{\alpha}_{NC}-{\alpha}^{*}),
\end{align}
which has an asymptotic chi-squared distribution with $d_{\alpha}$ degrees of freedom distribution. 
A consistent estimate of $\mathrm{Asvar}(\hat{\alpha}_{NC})$ can be obtained as
\begin{align*}
&\frac{m^{2}}{m_{1}m_{2}}\left(\hat{H}_{m}^{-1}-\begin{pmatrix} 1 & 0_{1\times d_{\theta}}\\ 0_{d_{\theta}\times 1} & 0_{d_{\theta}\times d_{\theta}}\end{pmatrix}\right),
\end{align*}
where ${H}_{m}$ is 
\begin{align*}
\frac{1}{m}\sum_{i=1}^{m}\begin{pmatrix} 1 & -\nabla_{\theta^{\top}} h(z_{i};\theta) \\ -\nabla_{\theta} h(z_{i};\theta) & \nabla_{\theta} h(z_{i};\theta)\nabla_{\theta^{\top}} h(z_{i};\theta) \end{pmatrix}
\frac{p(z_{i};\alpha)n(z_{i};\beta^{*})}{(\frac{m_{2}}{m} n(z_{i};\beta^{*})+\frac{m_{1}}{m} p(z_{i};\alpha))^{2}}|_{\alpha=\hat{\alpha}_{NC}},
\end{align*}
when $(x_{1},\cdots,x_{m_{1}},y_{1},\cdots,y_{m_{2}})=(z_{1},\cdots,z_{m})$. The estimator $\hat{H_{m}}$ is consistent because we have
\begin{align*}
H &= \mathrm{E}_{p^{*}}\left[\Omega(x) \frac{1}{\frac{m_{2}}{m}+\frac{m_{1}}{m}r^{*}(x)}\right]  \\
&= \int \Omega(x) \frac{n^{*}(x)p^{*}(x)}{\frac{m_{2}}{m}n^{*}(x)+\frac{m_{1}}{m}p^{*}(x)}\mathrm{d}\mu(x) \\
&= \int \Omega(x) \frac{n^{*}(x)p^{*}(x)}{\left(\frac{m_{2}}{m}n^{*}(x)+\frac{m_{1}}{m}p^{*}(x)\right)^{2}}\left(\frac{m_{2}}{m}n^{*}(x)+\frac{m_{1}}{m}p^{*}(x)\right)\mathrm{d}\mu(x).
\end{align*}
Testing for $\hat{\alpha}_{PL}$ can be performed in the same manner. However, the analytical form is complicated because of the plug-in. In this case, even if we use a Wald-type statistics \eqref{eq:wald}, the confidence interval is still valid, meaning that a nominal $100(1-\alpha$)\% has an actual coverage at least $100(1-\alpha)$\%. 
Score type-statistics and likelihood-ratio statistics can also be constructed for inference for $\alpha$.

\subsection{Optimal \texorpdfstring{$\phi(x;\gamma)$}{phi(x;gamma)} and \texorpdfstring{$n(x;\beta)$}{n(x;beta)} minimizing the asymptotic variance of \texorpdfstring{$\hat{\alpha}_{PL}$}{alpha(x;pl)}}

The reduction of variance by estimating $\beta$ is 
\begin{align}
\label{eq:value}
\frac{m_{1}}{m}A^{-1}BC^{-1}B^{\top}A^{-1}.
\end{align}
This depends strongly on the choice of auxiliary distribution models $\{n(x;\beta);\beta\in\Theta_{\beta}\}$ and $\{\phi(x;\gamma);\gamma \in \Theta_{\gamma}\}$. 
The minimum asymptotic variance of $\hat{\alpha}_{PL}$ and the optimal $\phi(x;\gamma)$ and $n(x;\beta)$ reaching that value are given as follows. 

 \begin{theorem}[Optimal $\phi(x;\gamma)$ and $n(x;\beta)$ minimizing the asymptotic variance of  $\hat{\alpha}_{PL}$]
\label{thm2}
We denote the dimension of $\beta$ as $d_{\beta}$ and $\alpha$ as $d_{\alpha}$. For any $n(x;\beta)$, $\phi(x;\gamma)$, the following inequality holds:
\begin{align*}
\mathrm{Asvar}[\hat{\alpha}_{PL}] \geq \frac{m^{2}}{m_{1}m_{2}}\left(\frac{m_{2}}{m}\mathrm{E}_{p^{*}}[\Omega(x)]^{-1}-\begin{pmatrix} 1 & 0_{1\times d_{\theta}}\\ 0_{d_{\theta}\times 1} & 0_{d_{\theta}\times d_{\theta}}\end{pmatrix}\right),
\end{align*}
where $\Omega(x)$ is given in Theorem \ref{thm2}.
When $n(x;\beta)=\exp(-h(x;\theta))$ around $\beta^{*}$ and $\theta^{*}$ and the function $\phi(x)$ is $\nabla_{\alpha}\log p(x;\alpha)$, which belong to the class $\mathcal{O}$ with $f(x)=x\log x$, the lower bound is achieved.
\end{theorem}

We note two points of Theorem \ref{thm2} in comparison with Theorem \ref{important}. First, in Theorem \ref{important}, the optimal $\phi(x)$ does not depend on the auxiliary distribution. However, when plugging-in, the optimal $\phi(x)$ does depends on the auxiliary distribution. Thus, the meaning of the optimality is different in Theorems \ref{thm2} and Theorem \ref{important}. Second, the optimal $f(x)$ is $x\log x$, which corresponds to the familiar KL divergence; in contrast, the optimal $f(x)$ is not $x\log x$ in Theorem \ref{important}. 

In practice, it is difficult to select $n(x;\beta)$ as in Theorem \ref{thm2}. The auxiliary models $n(x;\beta)$ must be chosen as in Section 3. Given auxiliary models $n(x;\beta)$, we suggest using plugging-in NCE rather than original NCE because plugging-in NCE always outperforms original NCE regardless of the choice of $\phi(x)$ as seed in Section 3. In addition, we recommend using $\phi(x)$ in Theorem \ref{important} because it is optimal in the original NCE, even if it is not optimal in the plugging-in NCE. This method has been shown empirically to have good performance as in Section 6.

\subsection{Special cases about asymptotic variances}

We have analyzed the asymptotic variances of $\hat{\alpha}_{NC}$ and $\hat{\alpha}_{PL}$ in the Theorem \ref{thm1}. We consider important special cases of asymptotic variances and touch on cases when $f(x)=x\log x$ or $n(x)$ is the true distribution. 
Finally, we see the asymptotic variances when $m_{2}/m$ approaches 1. 
From here, we will consider the asymptotic results scaled by $m_{1}$ rather than $m$. This is because we will compare the estimators obtained in the previous sections with the maximum likelihood estimator when the model is normalized.  

First, we consider the asymptotic variance in which the model is estimated using MLE, assuming that the normalizing constant is calculated explicitly. This is the best method among a particular broad class of estimators for parameters in regular parametric models in the sense of asymptotic variance \citep{van}.

\begin{corollary}[Normalized models]
\label{MLE}
When the normalizing constant \\ $\int\exp(-h(x;\theta)\mathrm{d}\mu(x)$
is calculated explicitly, the asymptotic variance of MLE estimator is $\frac{1}{m_{1}}\mathrm{var}\left[\nabla_{\theta^{\top}}  h^{*}(x)\right]^{-1}$.
\end{corollary}

Next, we return to unnormalized models again.
In this section, from here, we scale the variance by $m_{1}$ rather than $m$ to compare our estimators with those obtained using MLE. First, consider a case in which $f(x)$ is $x\log x$, that is, in which the associated divergence is Kullback-Leiber divergence. The asymptotic variance can be written as follows.

\begin{corollary}[Asymptotic variance when $f$ is Kullback-Leibler divergence]

When $f=x\log x$, the sequence $\sqrt{m_{1}}(\hat{\alpha}_{NC}-\alpha^{*})$ weakly converges to a normal distribution with mean zero and variance:
\begin{align*}
\frac{m}{m_{2}}\left(\mathrm{E}_{p^{*}}[\Omega(x)]^{-1}\mathrm{E}_{p^{*}}\left[\Omega(x) \left(\frac{m_{2}}{m}+\frac{m_{1}}{m}r^{*}(x)\right)\right]\mathrm{E}_{p^{*}}[\Omega(x)]^{-1}-
\begin{pmatrix} 1 & 0_{1\times d_{\theta}}\\ 0_{d_{\theta}\times 1} & 0_{d_{\theta}\times d_{\theta}}\end{pmatrix}\right).
\end{align*}
\end{corollary}

Next, consider the case where $n(x;\beta)=\exp(-h(x;\theta))$. This situation can be considered the ideal situation. 

\begin{corollary}[Asymptotic variance when $n(x;\beta)$ is the true distribution]
\label{col2}
When $\exp(-h(x;\theta))=n(x;\beta)$ in a neighborhood of $\theta^{*}=\beta^{*}$, the asymptotic variance $\frac{m_{1}}{m}A^{-1}GA^{-1}$ in Theorem \ref{important} and $\frac{m_{1}}{m}A^{-1}(G-\frac{\lambda}{1-\lambda}BC^{-1}B^{\top})A^{-1}$ in Theorem \ref{thm2} are written as follows:
\begin{align*}
\frac{m}{m_{2}}\left(\mathrm{E}_ {p^{*}}[\Omega(x)]^{-1}-
\begin{pmatrix} 1 & 0_{1\times d_{\theta}}\\ 0_{d_{\theta}\times 1} & 0_{d_{\theta}\times d_{\theta}}\end{pmatrix}\right),\\
\frac{m}{m_{2}}\left(\frac{m_{2}}{m}\mathrm{E}_{p^{*}}[\Omega(x)]^{-1}-
\begin{pmatrix} 1 & 0_{1\times d_{\theta}}\\ 0_{d_{\theta}\times 1} & 0_{d_{\theta}\times d_{\theta}}\end{pmatrix}\right),
\end{align*}
respectively, where $\Omega(x)$ is provided in the proof of the Theorem \ref{thm1}.
The term $\mathrm{E}_{p^{*}}[\Omega(x)]^{-1}$ can be written as
\begin{align}
\label{eq:wood}
\begin{pmatrix}\mathrm{E}_{p^{*}}[ \nabla_{\theta^{\top}}h^{*}](\mathrm{var}[\nabla_{\theta} h^{*}])^{-1} \mathrm{E}_{p^{*}}[\nabla_{\theta} h^{*}]+1 & -\mathrm{E}_{p^{*}}[\nabla_{\theta^{\top}} h^{*}]\mathrm{var}[\nabla_{\theta}h^{*}]^{-1} \\ -\mathrm{var}[\nabla_{\theta} h^{*}]^{-1}\mathrm{E}_{p^{*}}[\nabla_{\theta} h^{*}] & \mathrm{var}[\nabla_{\theta} h^{*}]^{-1} \end{pmatrix}.
\end{align}
\end{corollary}

The important facts derived from Corollary \ref{col2} are that the asymptotic variance of $\hat{\alpha}_{NC}$ does not reach the variance in the case of MLE shown in Corollary \ref{MLE} even if $n(y;\beta)=\exp(-h(x;\alpha))$. In contrast, the asymptotic variance of the estimator $\hat{\alpha}_{PL}$ is equal to the that of MLE when $n(y;\beta)=\exp(-h(x;\alpha)$, focusing on only $\theta$ and ignoring $c$ from \eqref{eq:wood}. It reaches the lower bound $\mathrm{var}[\nabla_{\theta} h^{*}]^{-1}$ in  Corollary 4.1.

In the extreme case in which we can obtain an infinite number of samples from the auxiliary distribution, the asymptotic variance of $\hat{\alpha}_{PL}$ and $\hat{\alpha}_{NC}$ becomes as follows, letting $m_{2}/m \to 1$ and $m_{1}/m \to 0$ . As a natural result, the asymptotic variances of the estimators are equal to that of MLE estimator. 

\begin{corollary}[Case of $m_{2}/m \to 1$]
\label{ideal}
When $m_{2}/m \to 1$, the asymptotic variances of $\hat{\alpha}_{NC}$ and $\hat{\alpha}_{PL}$become
\begin{align*}
\mathrm{E}_{p^{*}}[\Omega(x)]^{-1} -
\begin{pmatrix} 1 & 0_{1\times d_{\theta}}\\ 0_{d_{\theta}\times 1} & 0_{d_{\theta}\times d_{\theta}}\end{pmatrix}.
\end{align*}.
\end{corollary}

Finally, consider the meaning of the asymptotic variance in Corollary \ref{ideal}. The asymptotic variances of $c$ and $\theta$ have the meaning of the comparison of the two terms, $\mathrm{E}_{p^{*}}[\nabla_{\theta} h^{*}]\mathrm{E}_{p^{*}}[\nabla_{\theta^{\top}}h^{*}]$ and $\mathrm{E}_{p^{*}}[\nabla_{\theta} h^{*}\nabla_{\theta^{\top}}h^{*}]$. In fact, the asymptotic variance of $c$ is the ratio:
\begin{align*}
\frac{e}{1-e},
\end{align*}
where $e$ is defined by 
\begin{align*}
\mathrm{E}_{p^{*}}\left[\nabla_{\theta^{\top}}h^{*} \right]\mathrm{E}_{p^{*}}\left[\nabla_{\theta}h^{*} \nabla_{\theta^{\top}}h^{*}\right]^{-1}\mathrm{E}_{p^{*}}\left[\nabla_{\theta}h^{*}\right].
\end{align*}
Note that this ratio is equivalent to the (1,1) component in the right hand side of the \eqref{eq:wood} by the Woodbury formula. 
At the same time, the asymptotic variance of $\theta$ is equal to the difference,
\begin{align*}
\left( \mathrm{E}_{p^{*}}\left[\nabla_{\theta}h^{*} \nabla_{\theta^{\top}}h^{*}\right]-\mathrm{E}_{p^{*}}\left[\nabla_{\theta}h^{*} \right]\mathrm{E}_{p^{*}}\left[\nabla_{\theta^{\top}}h^{*} \right]\right)^{-1}.
\end{align*}

\section{Robustness}

\label{robust}

We found the optimal estimating equations minimizing the asymptotic variance of $\hat{\alpha}_{PL}$ and $\hat{\alpha}_{NC}$ in Theorems \ref{thm4} and \ref{thm2}. However, they do not have the property of robustness because gross-error sensitivity is not bounded as explained in later. Our aim here is to find the conditions for robustness. 

The influence function of an
estimator measures the effect on it of a small contamination at the point $x$, standardized by the mass of the contamination \citep{huber}. The supremum of influence functions over the
data-space, which is called called gross-error sensitivity, measures the worst influence of such contamination. A desirable robustness property for estimation is that the gross-error sensitivity is finite, meaning that the influence function is bounded.

Based on the above discussion, if \eqref{eq:vvv2} is bounded with respect to $\mathbf{x}$, we call these estimators robust. The following two conditions are sufficient conditions for the robustness of the estimator $\hat{\alpha}_{PL}$:
\begin{itemize}
\item Functions $\nabla_{\alpha} p(x;\alpha)$ and $ p(x;\alpha)/n(x;\beta)$, are bounded in $x$ on any $\alpha\in \Theta_{\alpha}$ and $\beta\in \Theta_{\beta}$,
\item $f''(x)$ is bounded in $x$ on $[0,M]$ for some $M\in (0,\infty)$ , 
\end{itemize}
because the estimating equations represented by \eqref{eq:vvv2} are bounded under these conditions. Note that these conditions are similar to the conditions for B-robustness of scoring rules \citep{DawidAlexander2014Taao}. Condition 5.2 in \cite{DawidAlexander2014Taao}, which is a sufficient condition for B-robustness, corresponds to the above two conditions. The first condition depends on the model and the auxiliary distribution. The second condition depends on the choice of the objective functions.
The objective function derived from the density power divergence, where $f(x)=x^{\beta+1}/(\beta+1)$ with $f''(x)=\beta x^{\beta-1}$, satisfies this when $\beta$ is greater than or equal to one. For example, $f(x)=0.5x^{2}$ is a specific example and the objective function with this $f$ has been written as \eqref{eq:ex2}. In contrast, the optimal $f(x)$ in Theorem \ref{thm4} does not satisfy the second condition because the second derivative $f''(x)= 1/x-m_{2}/(m_{2}+{m_{1}}x)$ is not bounded in a neighborhood of zero. Similarly, the optimal $f(x)$ in Theorem \ref{thm1} does not satisfy the condition because the second derivative $f''(x)=1/x$ is not bounded in a neighborhood of zero. This corresponds to the fact that MLE is not robust and that the estimation derived from density-power divergence is robust in the sense of influence functions \citep{basu}.

\section{Numerical Experiments}

In this section, we verify our statements experimentally based on two settings by discussing three points: (a) the proposed method, that is, that estimating parameters of auxiliary distributions using MLE, reduces the asymptotic variance; (b) that the objective function in Theorem 4.1 when $f(x) = x\log x - (\frac{m_{2}}{m_{1}}x+1)\log (1+\frac{m_{1}}{m_{2}}x)$, minimizes the asymptotic variance; and (c) that the objective function with chi-square divergence has the robust property. Note that we do not compare other popular methods such as score matching and contrastive divergence method because the simulation of the comparison to these methods has been already obtained \citep{noise}.

We did simulations in two settings: one-dimensional Gaussian distribution with unknown variance and truncated normal distribution with an unknown precision matrix. For the result of the former experiment, we explain in an . Here, we explain the latter experiment.  

We performed simulations to validate the two points (a) and (b) mentioned previously. Let $\mathrm{N}(x;0,D)\mathrm{I}(x>0.3)$ be a three-dimensional truncated normal distribution with mean $(0,0,0)$ and unknown precision matrix $D$ truncated below by $(0.3,0.3,0.3)$. As a result of the truncation, the matrix cannot be normalized analytically. The goal here is to estimate a precision matrix $D$. We set the true covariance matrix  $D^{-1}$ as 
\begin{align*}
\begin{pmatrix}
0.8  & 0.2 & 0.2 \\ 
0.2 & 0.8 & 0.2 \\
0.2 & 0.2 & 0.8
\end{pmatrix}. 
\end{align*}

There are several methods for estimating parameters in truncated normal distributions \citep{crain,truncated-normal}. Here, we consider NCE. To do this, we set the auxiliary distribution as $\mathrm{N}(x;0,\mathrm{I}_{3})\mathrm{I}(x>0.3)$ with mean $(0,0,0)$ and the precision matrix $\mathrm{I}_{3}$:
\begin{align*}
\begin{pmatrix}
1.0  & 0.0 & 0.0 \\ 
0.0 & 1.0 & 0.0 \\
0.0 & 0.0 & 1.0
\end{pmatrix},
\end{align*}
and we set $m_{1}:m_{2}=1:1$. Even if this is truncated, the normalizing constant of this auxiliary density is calculated exactly because the off-diagonal elements are zero and the distribution is regarded as the multiplication of ones of one-dimensional truncated normal distributions. We use R packages to generate samples and calculate the normalizing constant of the auxiliary distributions \citep{mv,tmvtnorm}.

\begin{figure}
    \centering
    \includegraphics[width=14cm]{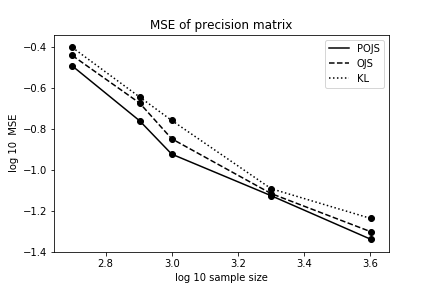}
    \caption{Comparison of MSE with respect to $f$}
    \label{fig:truncate-normal}
\end{figure}

We compare four methods: (POJS) plug-in NCE with optimal Jensen-Shannon divergence; (OJS) NCE with optimal Jensen-Shannon divergence; (Chi) NCE with chi-square divergence; and (KL) NCE with Kullback-Leibler divergence. The simulation is replicated independently with 100 times for each sample size. The results are shown in Figure \ref{fig:truncate-normal}. We report the Monte Carlo MSE of $\hat{D}$, changing the sample size as $(1000,1600,2000,4000,8000)$. For (POJS), we estimate the mean and variance of the auxiliary models using MLE. We do not list (Chi) because it is too larger compared with the other estimators. We see that plugging-in reduces the variance by comparing (POJS) and (OJS). It is also shown that the ranking of MSE is (JS), (KL) and (Chi) in ascending order. This supports the result of Theorem 4.1 that NCE with the optimal Jensen-Shannon divergence is most efficient .



\markboth{\hfill{\footnotesize\rm Masatoshi Uehara, Matsuda Takeru and Komaki Humiyasu} \hfill}
{\hfill {\footnotesize\rm } \hfill}

\bibhang=1.7pc
\bibsep=2pt
\fontsize{9}{14pt plus.8pt minus .6pt}\selectfont
\renewcommand\bibname{\large \bf References}
\expandafter\ifx\csname
natexlab\endcsname\relax\def\natexlab#1{#1}\fi
\expandafter\ifx\csname url\endcsname\relax
  \def\url#1{\texttt{#1}}\fi
\expandafter\ifx\csname urlprefix\endcsname\relax\def\urlprefix{URL}\fi

\bibliographystyle{chicago}      
\bibliography{main}   

\vskip .65cm
\noindent
Harvard University 
\vskip 2pt
\noindent
uehara\_m@g.harvard.edu
\vskip 2pt

\noindent
The University of Tokyo
\vskip 2pt
\noindent
matsuda@mist.i.u-tokyo.ac.jp \\
komaki@mist.i.u-tokyo.ac.jp

\

\appendix

\def\theequation{S\arabic{section}.\arabic{equation}}
\def\thesection{S\arabic{section}}

\fontsize{12}{14pt plus.8pt minus .6pt}\selectfont

\section{Proof of Theorems}

\subsection{Proof of Theorem 1}

Two conditions must be satisfied for the consistency of a Z-estimator, according to Theorem 5.9 in \cite{van}. The first condition: the uniform convergence
$\sup_{\gamma\in \Theta_{\gamma}} \|V_{m}(\mathbf{x},\mathbf{y};\gamma)-\mathrm{E}[V_{m}(\mathbf{x},\mathbf{y};\gamma)]]\|\to 0$
is guaranteed by the assumption (4).
What we must prove is the remaining condition, $\inf_{\epsilon\leq \|\gamma-\gamma^{*}\| }\|\mathrm{E}_{*}[V_{m}(\mathbf{x},\mathbf{y};\gamma)]\|>0$, the well-separated mode condition. The assumption of continuity of $V_{m}(\gamma)$ with respect to $\gamma$ from (3) leads to the continuity of $\mathrm{E}_{*}[V_{m}(\gamma)]$ with respect to $\gamma$ from the assumption (4). Combining the continuity of $\mathrm{E}_{*}[V_{m}(\gamma)]$ and assumption (2), the well-separated mode condition is reduced to be an identifiability condition:
\begin{align*}
\|\mathrm{E}_{*}[V_{m}(\mathbf{x},\mathbf{y};\gamma)]\|=0 \iff \alpha=\alpha^{*},\,\beta=\beta^{*}.
\end{align*}

The identifiability condition is proved by 
\begin{align*}
\|\mathrm{E}_{*}[V_{m}(\mathbf{x},\mathbf{y};\gamma)]\|=0 & \iff \|\mathrm{E}_{*}[V_{1m}(\mathbf{x},\mathbf{y};\alpha,\beta)]\|=0,\, \|\mathrm{E}_{*}[V_{2m}(\mathbf{y};\beta)]\|=0 \\
& \iff \|\mathrm{E}_{*}[V_{1m}(\mathbf{x},\mathbf{y};\alpha,\beta)]\|=0,\, \|\mathrm{E}_{n^{*}}[\nabla_{\beta}\log n(y;\beta)]\|=0 \\
& \iff \|\mathrm{E}_{*}[V_{1m}(\mathbf{x},\mathbf{y};\alpha,\beta)]\|=0,\, \beta = \beta^{*} \\
& \iff \alpha=\alpha^{*}, \beta = \beta^{*}.
\end{align*}

The logic from the second line to the third line is as follows. We know that $\mathrm{E}_{n^{*}}[\log n(x;\beta)]$ is uniquely minimized at $\beta=\beta^{*}$ because of Jensen inequality and assumption (1). From the convexity of $\beta\to\mathrm{E}_{n^{*}}[\log n(x;\beta)]$, a local minimum becomes a global minimum. This tells us that $\nabla_{\beta}\mathrm{E}_{*}[\log n(x;\beta)]=0\iff \beta=\beta^{*}$. From the exchangebility of expectation and differentiation to  from the assumption (4), we obtain $\mathrm{E}_{n^{*}}[\nabla_{\beta}\log n(x;\beta)]\iff \beta=\beta^{*}$.

The logic from the third line to the fourth line is as follows. We know that $\alpha \to d_{B}(p(x;\alpha^{*}),p(x;\alpha))$ is uniquely minimized at $\alpha=\alpha^{*}$ because of the property of $d_{B}(p(x;\alpha^{*}),p(x;\alpha))$ and assumption (1). Therefore, from the relationship $\nabla_{\alpha}d_{B}(p(x;\alpha^{*}),p(x;\alpha))=-\mathrm{E}_{*}[V_{1m}]$ and the convexity of $\alpha \to d_{B}(p(x;\alpha^{*}),p(x;\alpha))$, $\|\mathrm{E}_{*}[V_{1m}(\mathbf{x},\mathbf{y};\alpha,\beta^{*})]\|=0$ holds if and only if $\alpha=\alpha^{*}$. 

\begin{remark}
We have introduced the estimator as an M-estimator; however, in the proof, we regard the estimator as a Z-estimator. We can prove consistency even if we retain the original M-estimator form. In this case, the condition for proving the identifiability can be more relaxed. Specifically, we need not assume the convexity conditions in Theorem 1. This is because $\mathrm{E}_{*}[\mathrm{d}_{B}(p(x;\alpha^{*},p(x;\alpha)]$ is uniquely minimized when $\alpha=\alpha^{*}$. 
\end{remark}

\subsection{Proof of Theorem 2 }

For the asymptotic normality of a stratified Z-estimator, refer to \cite{wood}. The required conditions are satisfied from the assumptions; therefore, $\sqrt{m_{1}} (\hat{\gamma}_{PL}- \gamma^{*})$ converges to a normal distribution with mean zero and the variance
becomes $\Omega_{1}^{-1}\Omega_{2}{\Omega_{1}^{\top}}^{-1}$, where 
\begin{align}
\label{eq:sand}
\Omega_{1} &= \mathrm{E}_{*}\left[\nabla_{\gamma^{\top}}V_{m}(\mathbf{x},\mathbf{y})|_{\gamma^{*}}\right] \\
\Omega_{2} &= \frac{m_{1}}{m}\mathrm{var}_{p^{*}}\left[\left(\frac{m_{2}}{m}\phi(x)^{\top},0_{d_{\beta}\times 1}\right)^{\top}|_{\gamma^{*}}\right]+\\
&\frac{m_{2}}{m}\mathrm{var}_{n^{*}}\left[\frac{m_{1}}{m}\left(-\frac{p(y)}{n(y)}\phi(y)^{\top},-\nabla_{\beta}\log n(y;\beta)^{\top}\right)^{\top}|_{\gamma^{*}}\right].
\end{align}

First, we consider the term $\Omega_{1}$. The component $A\equiv \mathrm{E}_{*}[\nabla_{\gamma}V_{1m}(\mathbf{x},\mathbf{y};\gamma)]$ is calculated as follows:
\begin{align*}
A &= \sum_{i=1}^{m_{2}}\frac{m_{1}}{m^{2}}\mathrm{E}_{n^{*}}\left[-\phi(y_{i};\alpha,\beta)\frac{\nabla_{\alpha^{\top}}p(y_{i};\alpha)}{n(y_{i};\beta)}|_{\gamma^{*}}\right] \\
&=\frac{m_{1}m_{2}}{m^{2}}\mathrm{E}_{p^{*}}\left[-\phi(x;\alpha,\beta)\left(\nabla_{\alpha^{\top}} \log p(x;\alpha)\right)|_{\gamma^{*}}\right].
\end{align*}
Other components of $\mathrm{E}_{*}\left[\nabla_{\gamma^{\top}} V(\gamma)|_{\gamma^{*}}\right]$, that is, $B\equiv \mathrm{E}_{*}\left[\nabla_{\beta^{\top}} V_{1m}|_{\gamma = \gamma^{*}}\right]$ and $C\equiv \mathrm{E}_{*}\left[\nabla_{\beta^{\top}}V_{2m}|_{\gamma = \gamma^{*}}\right]$ are also calculated similarly. 
\begin{align*}
B &= \frac{m_{1}}{m^{2}}
\sum_{i=1}^{m_{2}}\mathrm{E}_{n^{*}}\left[ \phi(y_{i})\frac{p(y_{i})}{n(y_{i})}\nabla_{\beta^{\top}}\log n(y_{i};\beta)|_{\gamma^{*}} \right] \\
&= \frac{m_{1}m_{2}}{m^{2}}\mathrm{E}_{n^{*}}\left[\phi(x)\frac{p(x)}{n(x)}\nabla_{\beta^{\top}}\log n(x;\beta)|_{\gamma^{*}} \right], \\
C &= \mathrm{E}_{*}\left[\nabla_{\beta^{\top}} V_{2m}|_{\gamma = \gamma^{*}}\right] \\
&= -\frac{m_{1}m_{2}}{m^{2}}\mathrm{E}_{n^{*}}\left[\nabla_{\beta^{\top}} \nabla_{\beta}\log n(x;\beta)|_{\gamma^{*}}\right].
\end{align*}

Next, we consider the term $\Omega_{2}$. The $ (1,d_{\theta})\times (1,d_{\theta})$ block matrix, denoted by $G$, is calculated 
\begin{align*}
G &= \frac{m_{1}m_{2}}{m^{2}}\mathrm{var}_{p^{*}}\left[\phi(x;\alpha,\beta)|_{\gamma^{*}}\right]
+\frac{m_{1}m_{2}}{m^{2}}\mathrm{var}_{n^{*}}\left[\phi(x;\alpha,\beta )\frac{p(x;\alpha)}{n(x;\beta)}\right]|_{\gamma^{*}}\\
&= \frac{m_{1}m_{2}}{m^{2}}\mathrm{E}_{ p^{*}}\left[\phi(x)\phi(x)^{\top}\left(\frac{m_{2}}{m}+\frac{m_{1}}{m}\frac{p}{n}\right)|_{\gamma^{*}}\right]-\frac{m_{1}m_{2}}{m^{2}}\mathrm{E}_{p^{*}}[\phi(;\gamma)|_{\gamma^{*}}]\mathrm{E}_{p^{*}}\left[\phi(x;\gamma)^{\top}|_{\gamma^{*}}\right]
\end{align*}

The other components,the  $(d_{\alpha}+1,d_{\alpha}+d_{\beta})\times(d_{\alpha}+1,d_{\alpha}+d_{\beta})$ block matrices, and $(1,d_{\alpha})\times(d_{\alpha}+1,d_{\alpha}+d_{\beta})$ block matrix,  
are calculated similarly:
\begin{align*}
\frac{m_{2}}{m}\mathrm{var}_{n^{*}}\left[-\frac{m_{1}}{m}\nabla_{\beta} \log n(y;\beta)|_{\beta^{*}}\right] &= -\frac{m_{1}^{2}m_{2}}{m^{3}}\mathrm{E}_{n^{*}}\left[\nabla_{\beta^{\top}} \nabla_{\beta}\log n(x;\beta)|_{\gamma^{*}}\right]\\
&= \frac{m_{1}}{m}C, 
\end{align*}
and 
\begin{align*}
\frac{m_{2}}{m}\mathrm{cov}_{n^{*}}\left[-\frac{m_{1}}{m}\phi\frac{p(y;\alpha)}{n(y;\beta)}|_{\gamma^{*}},-\frac{m_{1}}{m}\nabla_{\beta}\log n(y;\beta)|_{\gamma^{*}}\right]=\frac{m_{1}}{m}B.
\end{align*}

From the \eqref{eq:sand} and the above calculations, the components of the asymptotic variance becomes
\begin{align*}
\Omega_{1}^{-1} =
    \begin{pmatrix}
       A & B  \\
       0_{d_{\theta}\times 1} & C 
    \end{pmatrix}^{-1}
    = 
    \begin{pmatrix}
       A^{-1} & -A^{-1}BC^{-1} \\
       0_{d_{\theta}\times 1} & C^{-1} 
    \end{pmatrix}, 
\Omega_{2} = \begin{pmatrix}
        G & \frac{m_{1}}{m}B  \\
        \frac{m_{1}}{m}B^{\top} & \frac{m_{1}}{m}C
     \end{pmatrix}.
\end{align*}
The upper left component of $\Omega_{1}^{-1}\Omega_{2}{\Omega_{1}^{\top}}^{-1}$, corresponding to the asymptotic variance of $\theta$, becomes 
\begin{align*}
A^{-1}\left(G-\frac{m_{1}}{m}BC^{-1}B^{\top}\right)A^{-1},
\end{align*}
noting that $A$ is a symmetric matrix. 

\begin{remark}
We apply results from work discussing stratified M-estimator \citep{wood}. However, there is a subtle differences in the setting. The objective function is different in each stratum in our case. On the other hand, they assume the same objective function for each stratum. Despite of this difference, their asymptotic results can be extended easily to our setting.  
\end{remark}

\begin{remark}
All of the conditions are standard conditions required for the asymptotic normality of a Z-estimator when the estimation equation is a differentiable function \citep{wood}. However, when $p(x;\alpha)$ is not differentiable with respect to $\alpha$ around $\alpha^{*}$, Theorem 1 cannot be applied directly. In this case, assuming and making use of the differentiability of $\gamma \to \mathrm{E}_{*}[V_{\gamma}(\mathbf{x},\mathbf{y};\gamma)]$, rather than assuming the differentiability of $\gamma \to V_{\gamma}(\mathbf{x},\mathbf{y};\gamma)$ directly, we can extend our results to non-differentiable objective functions. More precisely, let the model $p(x;\alpha)$ be differentiable in quadratic mean, and let the estimator be defined by replacing a function $\nabla_{\alpha}\log p(x;\alpha)$ with a score function derived from differentiability in quadratic mean, as mentioned earlier. Under the following conditions, the asymptotic normality holds:(1) $\{V_{m}(\mathbf{x},\mathbf{y};\gamma);\gamma \in \Theta_{\gamma}\}$ forms a Donsker family; (2) Same in Theorem 1; (3) $\gamma\to \mathrm{E}_{*}[V_{m}(\mathbf{x},\mathbf{y};\gamma)]$ is differentiable with respect to $\gamma$; and (4) $\gamma\to V_{m}(\mathbf{y},\mathbf{y};\gamma)$ is continuous when the range space is an $L_{2}(g^{*})$ with underlying density $g^{*}(x)$. \citep{van}.
\end{remark}

\subsection{Proof of Theorem 3}

The matrices $A$ and $G$ in Theorem 1 are written as 
\begin{align*}
A =& -\frac{m_{1}m_{2}}{m^{2}}\mathrm{E}_{p^{*}}\left[\begin{pmatrix} 1 & -\nabla_{\theta^{\top}} h^{*}(x)\\  \nabla_{\theta} h^{*} & \nabla_{\theta} h^{*}(x)\nabla_{\theta^{\top}} h^{*}(x) \end{pmatrix}
\psi(x)|_{\gamma^{*}}\right],\\
G =& \frac{m_{1}m_{2}}{m^{2}}\mathrm{E}_{p^{*}}\left[\begin{pmatrix} 1 & -\nabla_{\theta^{\top}} h^{*}(x)\\  \nabla_{\theta} h^{*} & \nabla_{\theta} h^{*}(x)\nabla_{\theta^{\top}} h^{*}(x) \end{pmatrix}
\left(\frac{m_{2}}{m}+\frac{m_{1}}{m}r^{*}(x)\right)\psi(x)^{2} |_{\gamma^{*}}\right]\\ &-\frac{m_{1}m_{2}}{m^{2}}\mathrm{E}_{p^{*}}[\phi(x)|_{\gamma^{*}}]\mathrm{E}_{p^{*}}[\phi(x)|_{\gamma^{*}}]^{\top}.
\end{align*}
The matrix $A^{-1}GA^{-1}$ becomes
\begin{align}
\label{eq:proof}
&\frac{m^{2}}{m_{1}m_{2}}\left(\mathrm{E}_{p^{*}}[\Omega
\psi|_{\gamma^{*}}]^{-1}
\mathrm{E}_{p^{*}}\left[\Omega \left(\frac{m_{2}}{m}+\frac{m_{1}}{m}r^{*}\right)\psi^{2}|_{\gamma^{*}} \right]
\mathrm{E}_{p^{*}}[\Omega
\psi|_{\gamma^{*}}]^{-1}-\Lambda\right),
\end{align}
where
\begin{align}
\Lambda &= \begin{pmatrix} 1 & 0_{1\times d_{\theta}}\\ 0_{d_{\theta}\times 1} & 0_{d_{\theta}\times d_{\theta}}\end{pmatrix}.
\end{align}

The term other than the constant in \eqref{eq:proof} is converted into the form 
\begin{align*}
\mathrm{E}_{p^{*}}[z_{1}(x)z_{2}(x)^{\top}]^{-1}\mathrm{E}_{p^{*}}[z_{2}(x)z_{2}(x)^{\top}]\mathrm{E}_{p^{*}}[z_{1}(x)z_{2}(x)^{\top}]^{-1},
\end{align*}
when $z_{1}(x)$ and $z_{2}(x)$ are set as  
\begin{align*}
z_{1}(x) &= \left(1, -\nabla_{\theta^{\top}} h^{*}(x)\right)^{\top}\frac{1}{\sqrt{\frac{m_{2}}{m}+\frac{m_{1}}{m}r^{*}(x)}},\\
z_{2}(x) &= \left(1, -\nabla_{\theta^{\top}} h^{*}(x)\right)^{\top} \sqrt{\frac{m_{2}}{m}+\frac{m_{1}}{m}r^{*}(x)}\psi(x)|_{\gamma^{*}}.
\end{align*}
A matrix extension of the Cauchy-Schwartz inequality \citep{matrix} yields that
\begin{align}
\label{eq:cau}
\mathrm{E}_{p^{*}}\left[z_{2}(x)z_{1}(x)^{\top}\right]^{-1}\mathrm{E}_{p^{*}}\left[z_{2}(x)z_{2}(x)^{\top}\right]\mathrm{E}_{p^{*}}[z_{1}(x)^{\top}z_{2}(x)]^{-1}\geq \mathrm{E}_{p^{*}}\left[z_{1}(x)z_{1}(x)^{\top}\right]^{-1}
\end{align}
and that the two sides of the inequality are equal when $z_{1}$ and $z_{2}$ satisfy $a^{\top}z_{1}(x)+b^{\top}z_{2}(x)=0$ for some non-zero $a\in \mathbb{R}^{d_{\alpha}},\,b\in \mathbb{R}^{d_{\alpha}}$ almost surely with respect to the measure $p^{*}(x)\mathrm{d}\mu(x)$. Note that $d_{\alpha}$ is the dimension of $\alpha$. This shows that $\psi(x)|_{\gamma^{*}}$ is proportional to 
\begin{align*}
\frac{1}{1+\frac{m_{1}}{m_{2}}r(x)}|_{\gamma^{*}}.
\end{align*}
By solving $f''(x)x=\frac{1}{1+\frac{m_{1}}{m_{2}}x}$, the function $f(x)$ becomes  $x\log x-\left(\frac{m_{2}}{m_{1}}+x\right) \log(1+\frac{m_{1}}{m_{2}}x)$.

Since the asymptotic variance is given by the inverse of the above Fisher information matrix, the statement is proved.

\subsection{Proof of Theorem 4}

First, we use the inequality:
\begin{align}
    \label{eq:cau5}
    \mathrm{E}_{n^{*}}\left[r^{2}\phi \phi^{\top}|_{\gamma^{*}}\right] &\geq
    \mathrm{E}_{n^{*}}\left[r\phi \nabla_{\beta^{\top}}n(y)|_{\gamma^{*}}\right]\mathrm{E}_{n^{*}}\left[\nabla_{\beta}n(y)\nabla_{\beta^{\top}}n(y)\right]^{-1}\mathrm{E}_{n^{*}}\left[r \nabla_{\beta^{\top}}n(y)\phi^{\top}|_{\gamma^{*}}\right] 
\end{align}
By multiplying $m_{1}m_{2}/m^{2}$, the left-hand size of this inequality is $BC^{-1}B^{\top}$. The statement can be proved by setting $z_{1}$ as $\nabla_{\beta}n(y)|_{\beta^{*}}$ and $z_{2}$ as $r^{*}(y)\phi(y;\gamma)|_{\gamma^{*}}$ and using a matrix extension of the Cauchy-Schwartz inequality:
\begin{align}
\label{eq:cau2}
\mathrm{E}_{n^{*}}\left[z_{2}(x)z_{2}(x)^{\top}\right]\geq \mathrm{E}_{n^{*}}\left[z_{2}(x)z_{1}(x)^{\top}\right]\mathrm{E}_{n^{*}}\left[z_{1}(x)z_{1}^{\top}(x)\right]^{-1}\mathrm{E}_{n^{*}}\left[z_{1}^{\top}(x)z_{2}(x)\right].
\end{align}
This inequality becomes equality if and only if  $a^{\top}\nabla_{\beta}\log n(y;\beta)+b^{\top}\frac{p(y)}{n(y)}\phi(y;\gamma)=0$ holds at $\gamma=\gamma^{*}$ for some non-zero $a\in \mathbb{R}^{d_{\beta}},b \in \mathbb{R}^{d_{\alpha}}$ .
From Theorem 3 and \eqref{eq:cau5}, we have  
\begin{align*}
&\mathrm{Asvar}[\hat{\alpha}_{PL}]  \\
&= A^{-1}\left(G-\frac{m_{1}}{m}BC^{-1}B^{\top}\right)A^{-1} \\
&\geq \frac{m^{2}}{m_{1}m_{2}}\left(\mathrm{E}_{p^{*}}[\Omega(x)
\psi|_{\gamma^{*}}]^{-1}
\mathrm{E}_{p^{*}}\left[\frac{m_{2}}{m}\Omega(x) 
\psi^{2}|_{\gamma^{*}} \right]
\mathrm{E}_{p^{*}}[\Omega(x)
\psi|_{\gamma^{*}}]^{-1}-
\Lambda \right).
\end{align*}
Let us apply the Cauchy-Schwartz inequality again:
\begin{align}
\label{eq:cau3}
\mathrm{E}_{p^{*}}\left[z_{2}(x)z_{1}(x)^{\top}\right]^{-1}\mathrm{E}_{p^{*}}\left[z_{2}(x)z_{2}(x)^{\top}\right]\mathrm{E}_{p^{*}}[z_{1}(x)^{\top}z_{2}(x)]^{-1}\geq \mathrm{E}_{p^{*}}\left[z_{1}(x)z_{1}^{\top}(x)\right]^{-1},
\end{align}
setting 
\begin{align*}
z_{1}(x) = \left(1, -\nabla_{\theta^{\top}} h^{*}(x)\right)^{\top},\, z_{2}(x) = \left(1, -\nabla_{\theta^{\top}} h^{*}(x)\right)^{\top} \psi(x)|_{\gamma^{*}}.
\end{align*}
Both sides of the inequality \eqref{eq:cau3} are equal if and only if when $\psi(x)$ is constant.
Therefore, we have 
\begin{align}
\label{eq:inq}
\mathrm{Asvar}[\hat{\alpha}_{PL}] \geq \frac{m^{2}}{m_{1}m_{2}}\left(\frac{m_{2}}{m}\mathrm{E}_{p^{*}}\left[z_{1}(x)z_{1}^{\top}(x)\right]^{-1}-\begin{pmatrix} 1 & 0_{1\times d_{\theta}}\\ 0_{d_{\theta}\times 1} & 0_{d_{\theta}\times d_{\theta}}\end{pmatrix}\right).
\end{align}
Notice that the right hand side of inequality \eqref{eq:inq} does not depend on $\phi^{*}$ and $n^{*}$. Hence, this term becomes the lower bound of $\mathrm{Asvar}[\hat{\alpha}_{PL}]$ when we can choose $\phi$ over the class $\mathcal{L}$ and $n(x;\beta)$.
Finally, consider the conditions for verifying that the two sides in inequality \eqref{eq:inq} are equal. For equality, conditions  $\psi(x)=1$ and $a^{\top}\nabla_{\beta}\log n(x;\beta)+b^{\top}\frac{p(x)}{n(x)}\phi(x;\gamma)|_{\gamma^{*}}=0$ are required. From the definition $\psi(x)=f''(r(x))r(x)$, these conditions are satisfied when $f(x)=x\log x$ and $n(x;\beta)=\exp(-h(x;\theta))$ around $\beta^{*}=\theta^{*}$.

\subsection{Proof of Corollary 1}

The Fisher information matrix is given by 
\begin{align*}
\mathrm{E}_{p^{*}}\left[\nabla_{\theta} \overline{\exp(-h(x;\theta))}\nabla_{\theta^{\top}} \overline{\exp(-h(x;\theta))} \right].
\end{align*}

\subsection{Proof of Corollary 2}

When $f=x\log x$, $f''(x)$ becomes $\frac{1}{x}$; therefore, $\psi(x)$ is one from the definition of $\psi(x)$. Using Theorem 3 and multiplying by $m_{1}/m$ proves the statement. 

\subsection{Proof of Corollary 3}

Using the equation $r(x;\alpha^{*}.\beta^{*})=1$, the former statement is proved by applying it to Theorem 3 and 4 directly. The latter statement is proved using the Woodbury formula.

\subsection{Proof of Corollary 4}

The statement is proved by using Theorem 3 and 4, letting $m_{2}/m=1$.

\section{View according to influence functions}

For a clearer understanding, we focus on calculating explicitly from a Taylor series expansion without relying directly on the formula. We have
\begin{align}
&\sqrt{m}(\hat{\alpha}_{PL}-\alpha^{*})  \\
=&\frac{-1}{\sqrt{m}} \mathrm{E}_{*}\left[\nabla_{\alpha^{\top}} V_{1m}|_{\gamma^{*}}\right]^{-1}\left(V_{1m}^{*}- \mathrm{E}_{*}\left[\nabla_{\beta^{\top}}V_{1m}|_{\gamma^{*}}\right]\mathrm{E}_{*}\left[\nabla_{\beta^{\top}}V_{2m}|_{\gamma^{*}}\right]^{-1}V_{2m}^{*} \right) +\mathrm{o}_{p}(1) \\
=& \frac{-1}{\sqrt{m}} A^{-1}\left(V_{1m}^{*}- \mathrm{E}_{*}\left[V_{1m}V_{2m}^{\top} |_{\gamma^{*}}\right]\mathrm{E}_{*}\left[V_{2m}V_{2m}^{\top}|_{\gamma^{*}}\right]^{-1}V_{2m}^{*} \right) +\mathrm{o}_{p}(1),
\label{eq:above}
\end{align} 
where $V_{1m}^{*}=V_{1m}(\mathbf{x},\mathbf{y};\gamma^{*})$ and 
$V_{2m}^{*}=V_{2m}(\mathbf{y};\gamma^{*})$. 
The second line to the third line is based on the relation $\mathrm{E}_{*}\left[\nabla_{\beta}V_{2m}|_{\gamma^{*}}\right]=-\mathrm{E}_{*}\left[V_{2m}V_{2m}^{\top}|_{\gamma^{*}}\right]=C$ and $\mathrm{E}_{*}\left[\nabla_{\beta^{\top}}V_{2m}|_{\gamma^{*}}\right]=-\mathrm{E}_{*}\left[V_{1m}V_{2m}^{\top} |_{\gamma^{*}}\right]=B$ as explained in the proof of Theorem 2.
The term $\sqrt{m}(\hat{\alpha}_{NC}-\alpha^{*})$ is equal to the term:
\begin{align}
\label{eq:above3}
 -\mathrm{E}_{*}\left[\nabla_{\alpha^{\top}} V_{1m}|_{\gamma^{*}}\right]^{-1}V_{1m}(\mathbf{x},\mathbf{y};\gamma^{*})+ \mathrm{o}_{p}(1).
\end{align}
The middle part of the right hand side of \eqref{eq:above} is the projection of the term in \eqref{eq:above3} to an orthogonal space, which is spanned by $V_{2m}$ in the Hilbert space consisting of square-integrable functions with mean zero. 
The length of influence functions, which correspond to  variances, is shortened by the projection. This explains geometrically why the variance is reduced. Equations \eqref{eq:above3} and \eqref{eq:above} yield the following statement.

\begin{corollary}
\begin{align}
\label{eq:asym3}
\sqrt{m}(\hat{\alpha}_{NC}- \hat{\alpha}_{PL}) &\stackrel{d}{\longrightarrow } \mathrm{N}\left(0,\frac{m_{1}}{m}A^{-1}BC^{-1}B^{\top}A^{-1} \right). 
\end{align}
\end{corollary}
\begin{proof}
The statement is obtained by subtracting \eqref{eq:above3} from \eqref{eq:above}.
\end{proof}

\section{One-dimensional Gaussian distribution with unknown variance}

For simplicity, consider a one-dimensional Gaussian distribution. Note that this model can be normalized easily in closed form. The reason why we use such a simple model is that we can easily see the validity of the three points (a), (b) and (c) mentioned in Section 7. 

Let $\exp(-\theta x^{2})$ be the posited unnormalized model. The one-parameter extended model in NCE is $\exp(c-\theta x^{2})$. We set the true $\theta$ and $c$ as $0.5$ and $-\log({\sqrt{2\pi}})$, respectively. We run NCE to estimate $c$ and $\theta$. We use two auxiliary models: (Close) a normal distribution with mean $0.2$ and variance $1.2$, (Far) a normal distribution with mean $1.6$ and variance $1.2$. In the former case, the overlapping region of target distribution and the auxiliary distribution is small; in the latter case, both are overlapping. We also set $m_{1}:m_{2}=1:2$.

We calcualte Monte Carlo MSEs comparing five situations:(POJS) plug-in NCE with optimal Jenssen-Shannon divergence ($f=x\log x-(x+\frac{m_{2}}{m_{1}})\log(1+\frac{m_{1}}{m_{2}}x$); (OJS) NCE with optimal Jenssen-Shannon divergence; (JS) NCE with Jenssen-Shannon divergence ($f=x\log x-(1+x)\log(1+x)$); (KL) NCE with Kullback-Leibler divergence $(f=x\log x)$; and (Chi) NCE with chi-square divergence $(f=0.5x^{2})$. We calculate mean square errors (MSE). The simulation is replicated $300$ times for each sample size. 

\begin{figure}
\centering
\includegraphics[width = 13cm]{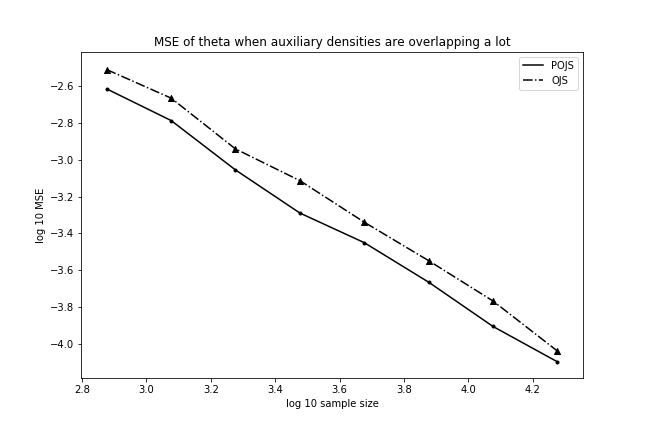}
\includegraphics[width = 13cm]{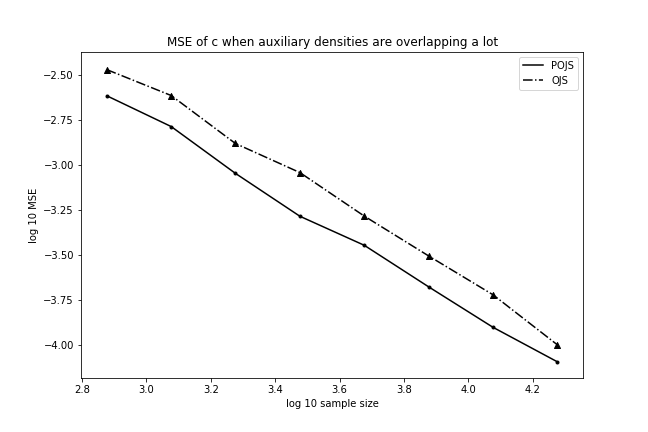}
\caption{Comparison of MSE plug-in or not when two densities overlap a lot, that is, auxiliary distribution is (Close).}
\label{fig:var1}
\end{figure}

\begin{figure}
\centering
\includegraphics[width = 13cm]{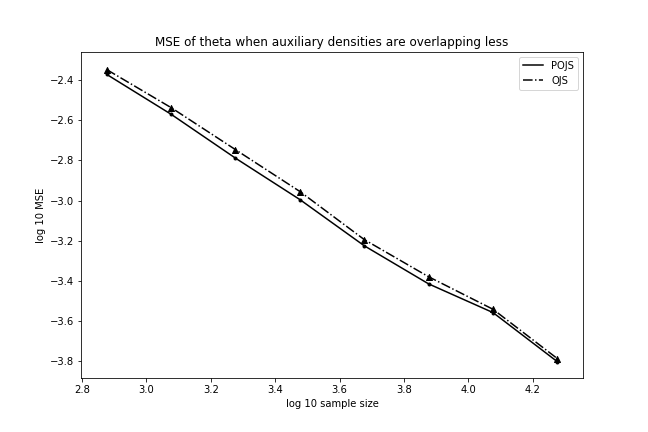}
\includegraphics[width = 13cm]{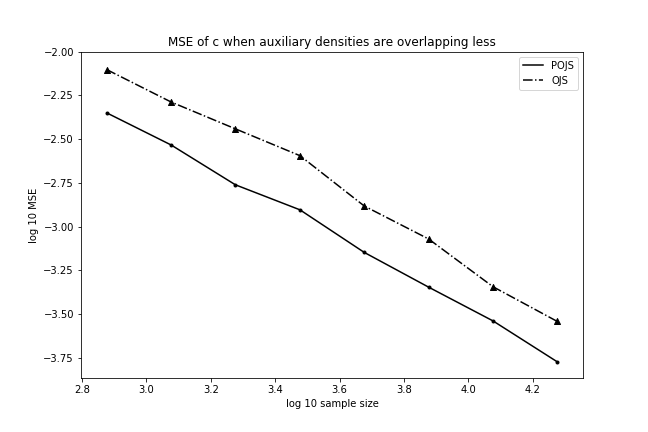}
\caption{Comparison of MSE plug-in or not when two densities overlap less, that is, auxiliary distribution is (Far).}
\label{fig:var2}
\end{figure}

First, to confirm the point (a), see Figures \ref{fig:var1} and Figure \ref{fig:var2}. Note that the both the $x$ scale and $y$ scales are log-scales. We compare the MSE of $\hat{c}$ and $\hat{\theta}$ when plug-in case (POJS) with not plug-in case (OJS), using the sample sizes as $753,1194,1890,3000,4752,7533,11943,18927$.
In the plug-in case, we estimate the mean and variance of the auxiliary models using MLE. It is shown in Figure \ref{fig:var1}  that estimating the parameters of the auxiliary distribution and plug-in is beneficial because the MSE is reduced significantly. The same thing can also be confirmed in Figure \ref{fig:var2}, but the difference is less. This suggests that plug-in is effective when two density overlap more.

\begin{figure}
\centering
\includegraphics[width = 13cm]{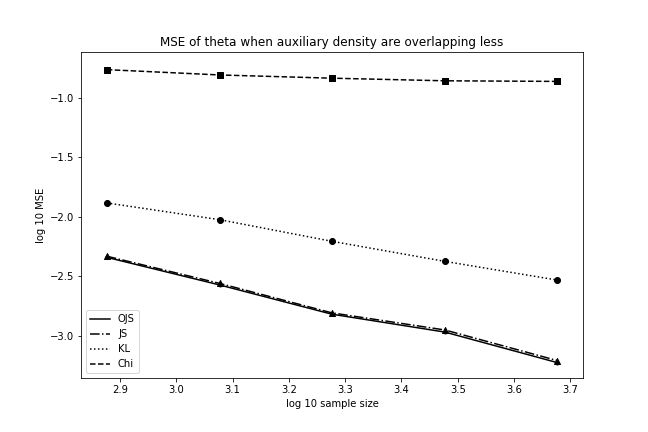}
\includegraphics[width = 13cm]{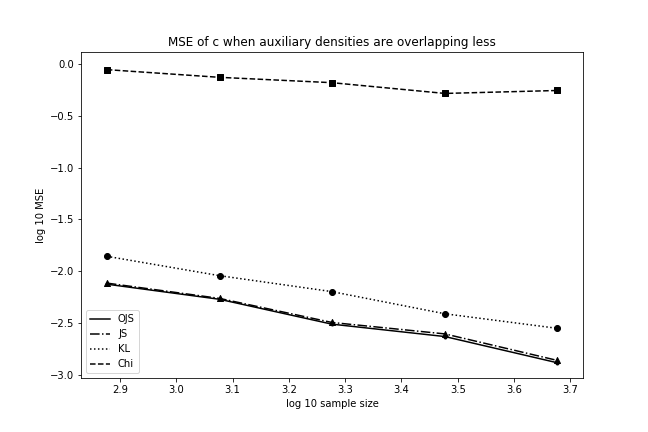}
\caption{Comparison of MSE with respect to $f$ when two densities are overlapping less, that is, auxiliary distribution is (Far). The two lines (OJS) and (JS) appear to be overlapping.}
\label{fig:var3}
\end{figure}

\begin{figure}
\centering
\includegraphics[width = 13cm]{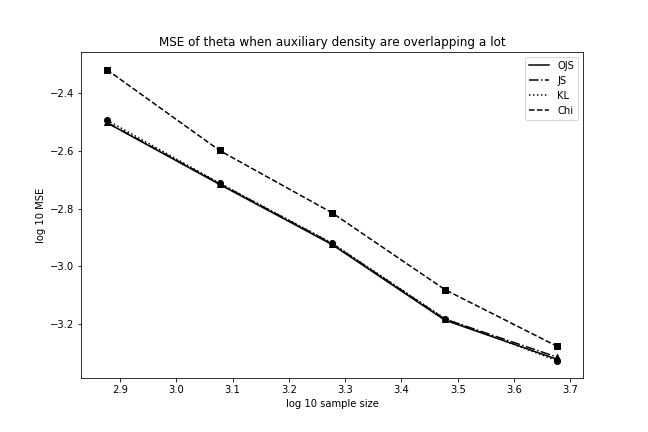}
\includegraphics[width = 13cm]{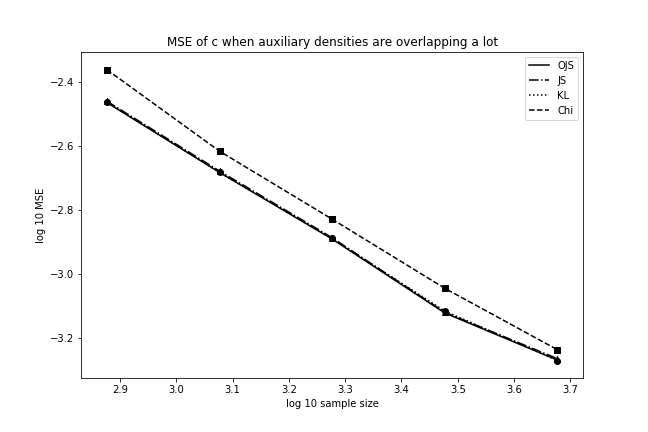}
\caption{Comparison of MSE with respect to $f$ when two density are overlapping substantially, that is, auxiliary distribution is (Close). Three lines (OJS), (JS) and (KL) appear to be overlapping.}
\label{fig:var4}
\end{figure}

\begin{figure}
\centering
\includegraphics[width = 13cm]{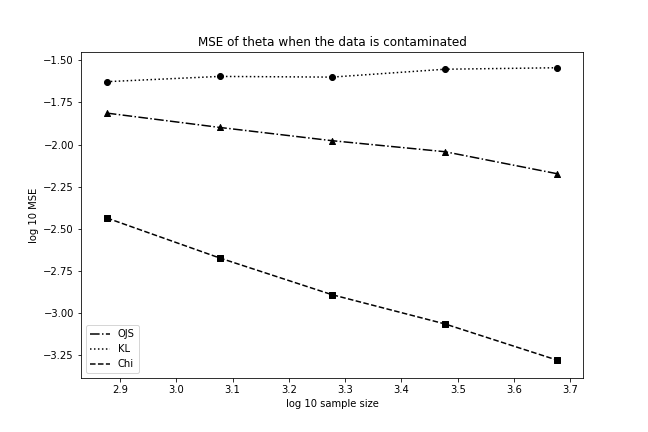}
\includegraphics[width = 13cm]{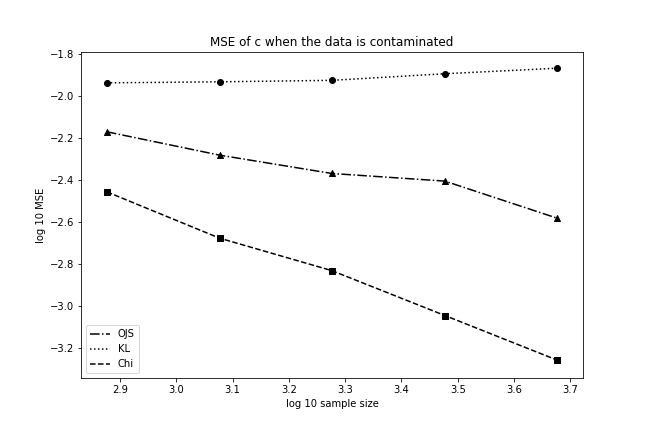}
\caption{Comparison of MSE with respect to $f$ when there is noise contamination and two densities are overlapping a lot, that is, auxiliary distribution is (Close).  }
\label{fig:robust}
\end{figure}

Next, to confirm the point (b), see Figures \ref{fig:var3} and \ref{fig:var4}. We compare the MSE of $\hat{c}$ and $\hat{\theta}$ with regard to the form of the objective functions, specifically, (POJS), (OJS), (KL), and (Chi), changing sample size as previously. The ranking of the MSE is (POJS), (OJS), (KL), and (Chi) in ascending order for any sample size. This result matches our analysis in Theorem 4.1 that the asymptotic variance of NCE with optimal Jensen-Shannon divergence takes the smallest value among a particular class of estimators. There are two things to note here. First, the difference of (OJS) and (JS) is small. This is because (JS) is optimal when $m_{1}:m_{2}=1:1$ and still has good performance even if $m_{1}:m_{2}=1:2$. It can be confirmed that (OJS) always outperforms (JS); however, the difference is seen only slightly in the graph. Second, when the overlapping of the target density and the auxiliary densities is less, the difference appears to be more significant. In fact, we can see the difference well in Figure \ref{fig:var3} when the auxiliary distribution is (Far), but we cannot see the difference well in Figure \ref{fig:var4} when the auxiliary distribution is (Close). 

We also confirm the validity of the asymptotic variance in Theorem 4.1. We calculate MSE and then multiply the result by the sample size when sample size is $10000$ and each simulation is replicated $1000$ times. We compare this result with the analytical form represented in Theorem 4.1. When the auxiliary distribution is (Close), the former is $(2.08,2.33)$ and the latter is $(2.08,2.34)$. When the auxiliary distribution is (Close), the former is $(3.14,6.43)$ and the latter is $(3.16,6.43)$. This supports the validity of Theorem 4.1 and confirms that the method used to construct confidence intervals is valid. 

Finally, we confirm the point (c). Consider the case in which the data distribution is contaminated with some noise. When the true data generating process is the same as it was previously and the auxiliary distribution is (Close), we added one outlier with value $200$ into the samples. The ranking according to MSE is (Chi), (POJS) and (KL) in ascending order for any sample size. We performed simulation according to the form of $f$ for each sample size. The Monte Carlo MSE for 300 replications is shown in Figure \ref{fig:robust}. This suggest that NCE with a chi-square divergence is robust because its MSE is smaller than that of the others. This result is consistent with our analysis in Section 6.

\end{document}